\documentclass{article}

 \usepackage[main, final]{neurips_2026}

\usepackage[utf8]{inputenc} 
\usepackage[T1]{fontenc}    
\usepackage{hyperref}       
\usepackage{url}            
\usepackage{booktabs}       
\usepackage{amsfonts}       
\usepackage{nicefrac}       
\usepackage{microtype}      
\usepackage{xcolor}         

\usepackage{amsmath}
\usepackage{amssymb}
\usepackage{mathtools}
\usepackage{amsthm}

\usepackage{algorithm}
\usepackage{algorithmic}
\usepackage{microtype}
\usepackage{bbm}
\usepackage{hyperref}
\usepackage{booktabs}   
\usepackage{tabularx}   
\usepackage{capt-of}  
\usepackage{subcaption}
\usepackage{wrapfig}

\usepackage{enumitem}

\newtheorem{remark}{Remark}
\theoremstyle{plain}
\newtheorem{theorem}{Theorem}
\newtheorem{proposition}{Proposition}
\newtheorem{corollary}{Corollary}
\newtheorem{lemma}{Lemma}

\theoremstyle{definition}
\newtheorem{definition}{Definition}
\newtheorem{assumption}{Assumption}

\usepackage{cleveref}
\crefname{equation}{Eq.}{Eqs.}
\crefname{assumption}{Assumption}{Assumptions}
\crefname{definition}{Definition}{Definitions}
\crefname{theorem}{Theorem}{Theorems}
\crefname{lemma}{Lemma}{Lemmas}
\crefname{proposition}{Proposition}{Propositions}
\crefname{remark}{Remark}{Remarks}
\crefname{corollary}{Corollary}{Corollaries}

\title{Provable Test-Time Scaling for Beam Search in LLM Reasoning}

\author{%
  Qijia He\thanks{equal contribution} \\
  The Ohio State University\\
  \texttt{he.2806@osu.edu} \\
  \And
  Yu Huang$^*$ \\
  University of Pennsylvania \\
  \texttt{yuh42@wharton.upenn.edu} \\
  \AND
  Yuan Cheng$^*$ \\
  National University of Singapore \\
  \texttt{yuan.cheng@u.nus.edu} \\
  \And
  Yuxin Chen \\
  University of Pennsylvania \\
  \texttt{yuxinc@wharton.upenn.edu} \\
  \And
  Yingbin Liang \\
  The Ohio State University\\
  \texttt{liang.889@osu.edu} \\
}

\begin{document}

\maketitle

\begin{abstract}
Beam-search–based test-time methods provide an effective way to improve large language model (LLM) performance on long-horizon generation by pruning invalid reasoning paths early, leading to significantly improved reasoning efficiency and more favorable test-time cost scaling. Despite strong empirical success, the theoretical understanding of beam search remains limited. In this paper, we study the test-time compute guarantee of the commonly used beam search framework that uses the model's internal log-likelihood for intermediate scoring, while relying on an external reward model only after a complete response is generated. We first establish a lower bound for vanilla beam search, showing that at least $\Omega(C^\star(x)^2)$ samples are required for the optimal response to survive, where $C^\star(x)$ is the token-level coverage coefficient for prompt $x$.
This motivates our modified confidence-filtered beam search (CF-Beam), which reduces the sufficient coverage dependence from quadratic to nearly linear under prefix competitiveness, for fixed horizon, gap, and target accuracy. We then show that the regret of CF-Beam is upper-bounded by the probability of rare failure events and the reward estimation error scaled by a path-level coverage coefficient, where the rare-failure term vanishes as per-step sampling increases. Our results highlight a fundamental advantage of beam search over sequence-level inference methods such as \textit{Best-of-N} and \textit{Best-of-Majority}. While the guarantees of these approaches typically involve coverage coefficients that grow exponentially with the horizon $L$, CF-Beam controls the dominant search-induced term through a token-level coverage coefficient that scales polynomially with $L$. Our numerical experiments further confirm that beam search is more robust on hard instances and under increasing reasoning horizons.

\end{abstract}

\section{Introduction}

Test-time computing has emerged as a powerful paradigm for improving the performance of large language models (LLMs) by allocating additional computation at inference time. Unlike post-training approaches such as supervised fine-tuning and reinforcement learning from human feedback, which improve performance by updating the model’s parameters, test-time methods operate on a fixed model and improve accuracy via repeated sampling, aggregation, and selection. Due to their simplicity in implementation and striking performance scaling behavior, test-time methods have been widely deployed in real-world systems \citep{compound-ai-blog, team2023gemini}. 

Existing and widely studied test-time methods include Best-of-$N$ sampling (BoN), majority voting, chain-of-thought (CoT) reasoning, and their corresponding variants \citep{jinnai2024regularized, ichihara2025evaluation,wang2022self, liu2025trust, di2025best}. Despite their empirical success, applying these approaches to complex real world tasks can be prohibitively expensive, as BoN and voting-based schemes require a rapidly increasing number of samples to maintain accuracy as context length grows. Moreover, the presence of imperfect verifiers can exacerbate selection errors by over-trusting spurious candidates or discarding correct but low-scoring solutions \citep{huang2025best}. 


To improve the reasoning capabilities of LLMs on long reasoning chains and mitigate the compounding errors in final answers that arise from accumulated small mistakes or model imperfections across intermediate steps \citep{xu2025collaborative, xie2023self}, \emph{beam search} has been found to be a promising test-time approach for sparse-success and long-horizon settings. This method expands candidate responses token by token and applies a step-level verifier, which can be either an external reward model \citep{lightman2023let, setlur2024rewarding} or an internal self-evaluation metric \citep{xie2023self}, to prune low-quality continuations at each step. By aggregating candidate trajectories and pruning inferior paths, beam search allocates inference-time computation to the most promising prefixes toward desired outputs, thereby increasing the probability of identifying correct solutions \citep{yu2025scaling, zhu2024deductive} while simultaneously reducing unnecessary test-time computation.

Despite a rich body of empirical evidence demonstrating the effectiveness of beam search, theoretical analysis of beam search remains limited. \citet{DBLP:conf/emnlp/MeisterCV20} studied what implicit objective function beam search optimizes. Recently, \cite{fadeeva2025don} studied beam search for self-consistency-based uncertainty quantification. To the best of our knowledge, there has not been
a principled theoretical understanding of its test-time scaling performance, which raises the following fundamental research questions:

\begin{enumerate}[itemsep=0pt, parsep=0pt, topsep=0pt,leftmargin=*]
\item[1.] {\em How does beam search scale at test time, particularly when test-time compute is constrained and verifiers are imperfect?} 

\item[2.] {\em Compared to recent test-time guarantees for Best-of-Majority \cite{di2025best} and best-of-N \cite{huang2025best}, does beam search enjoy provable advantages?}
\end{enumerate}
From the analysis standpoint, the main technical obstacle is that step-by-step pruning in beam search induces strong token-level dependencies over time, making the resulting dynamics harder to analyze than in sequence-level methods. Moreover, beam search often relies on heuristic intermediate verifiers whose step-wise scores are not aligned with a single, globally consistent objective. 





In this paper, we answer the above questions by analyzing the test-time performance of beam search in a sampling-access setting, where token probabilities are estimated from sampled continuations and used as an internal verifier, while an external reward model is accessible only after the complete response is generated.
\begin{itemize}[leftmargin=*, itemsep=2pt, topsep=2pt, parsep=0pt, partopsep=0pt]
    \item \textbf{Vanilla-Beam vs.~CF-Beam.} We establish a lower bound of sample complexity, showing that vanilla beam search (Vanilla-Beam) requires the sample size to be at least $\Omega(C^\star(x)^2)$ for the optimal response to survive, where $C^\star(x)$ is the token-level coverage coefficient for prompt $x\in\mathcal{X}$. This motivates us to filter the low-probability samples via the estimated sampling policy, resulting in the confidence-filtered beam search (CF-Beam) as shown in Figure~\ref{fig:beam search pipeline}, which reduces the sufficient coverage dependence from quadratic to nearly linear under prefix competitiveness, for fixed horizon, gap, and target accuracy.
    
        

    \item \textbf{Regret Analysis.} We show that the regret of CF-Beam is upper bounded by the probability of rare failure events and the reward estimation error scaled by a path-level coverage coefficient, where the rare-failure term vanishes as per-step sampling increases. Our results highlight a fundamental advantage of beam search over sequence-level inference methods such as \textit{Best-of-$N$} and \textit{Best-of-Majority}. While the guarantees of those approaches involve coverage coefficients that may potentially grow {\em exponentially} with the horizon $L$, CF-Beam controls the dominant search-induced term through a token-level coverage coefficient that scales only {\em polynomially} with $L$.
    
    \item \textbf{Reward-Model-Free Regret.} We further analyze the regret for a fully  self-consistent CF-Beam. In contrast to standard CF-Beam, whose guarantees depend on the error of an external reward model, the regret bound for the self-consistent version is driven solely by the probability of rare failure events. 
    Since the regret is decoupled from the reward-model error,
    self-consistent CF-Beam admits a sufficient total query budget growing only \textit{linearly} in $L$, up to logarithmic factors, for fixed beam width, token-level coverage, prefix gap, and target accuracy.

    \item \textbf{Numerical experiments.} Our numerical results compare CF-Beam against vanilla beam search, BoN, Majority voting and Best-of-Majority, demonstrating that CF-Beam achieves better performance on hard instances and is more robust to increasing horizon length.
\end{itemize}




\section{Preliminaries}\label{sec: preliminary}

\paragraph{Problem Setting and Decoding.} Test-time algorithms are decoding strategies that use extra inference-time compute to map an input prompt under a fixed LLM model to an improved output response. Formally, given a prompt $x\in \mathcal{X}$, the goal of the algorithm is to produce a response $y$, which is represented as a length-$L$ token sequence
$y=(a_1,\ldots,a_L)\in\mathcal{A}^L$, based on a fixed reference LLM policy
$\pi_{\text{ref}}: \mathcal{X}\to \Delta(\mathcal{A}^L)$. Here, $\mathcal{A}$ denotes the token vocabulary, and $\Delta(\mathcal{A}^L)$ is the set of probability distributions over $\mathcal{A}^L$. 

\begin{wrapfigure}{r}{0.55\linewidth}       
    \centering
    \includegraphics[width=1\linewidth]{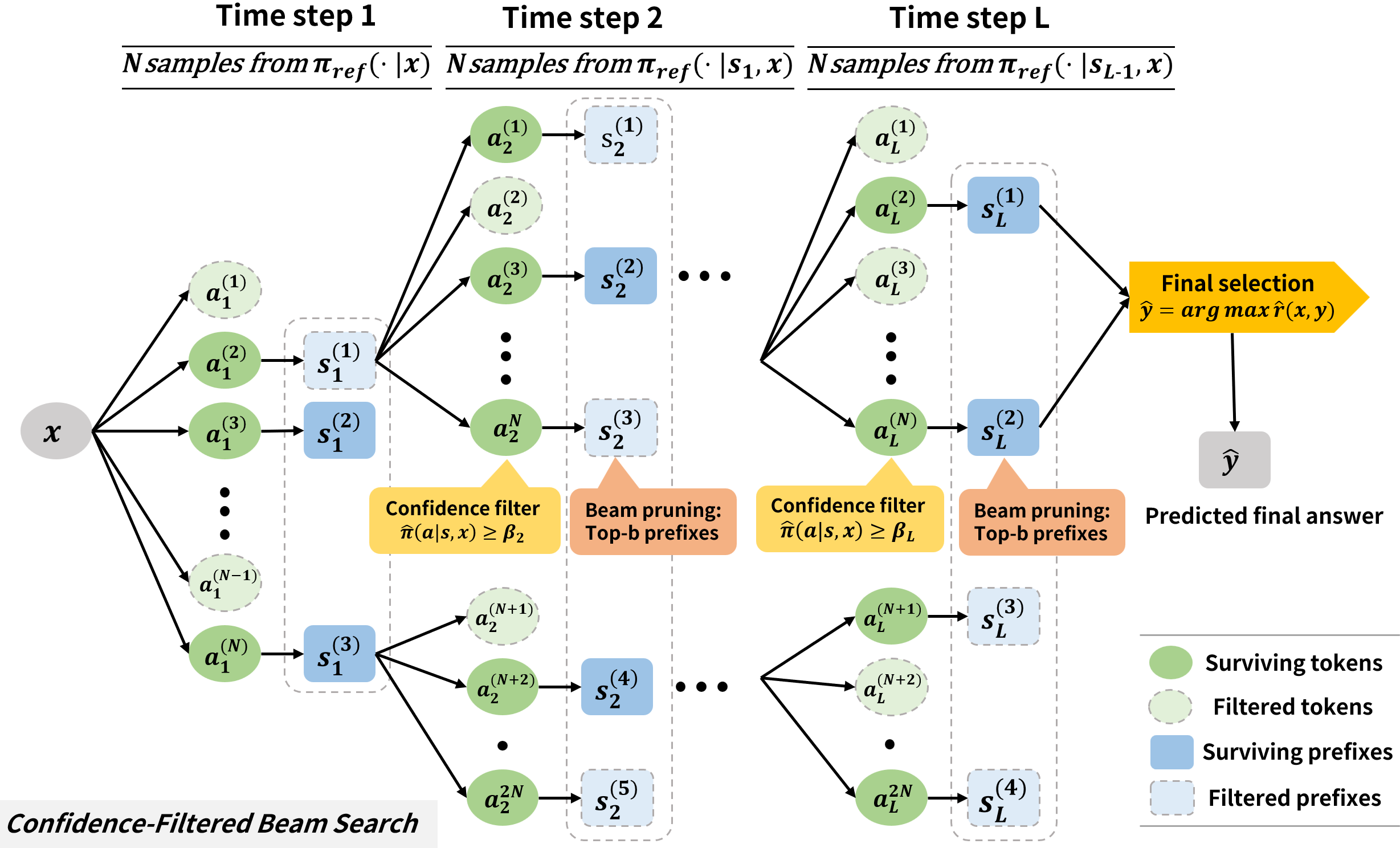}  
    \caption{Confidence-Filtered Beam Search Pipeline}
    \label{fig:beam search pipeline}
\end{wrapfigure}

The quality of the complete response $y$ is evaluated by a \textit{ground-truth} outcome reward $r^\star : \mathcal{X} \times \mathcal{Y} \to [0, 1]$, which is only revealed after the entire response is generated. Such sparse outcome rewards are standard in LLM settings, e.g., in reinforcement learning with verifiable feedback (RLVR) \citep{chen2025outcome, he2025response}. In practice, $r^\star$ is typically unavailable, and the algorithm designer instead relies on an approximate reward model $\hat{r} : \mathcal{X} \times \mathcal{Y} \to [0, 1]$ that assigns a surrogate score $\hat{r}(x, y)$ to assess the quality of response $y$ for any prompt $x$.

The reference policy $\pi_{\text{ref}}$ further admits a standard autoregressive factorization by steps. Let $s_{t-1}:=(a_1,\ldots,a_{t-1})$ denote the token prefix generated up to step $t-1$
(with $s_0=\emptyset$). Then we have
\begin{equation*}
\pi_{\mathrm{ref}}(y\mid x)=\prod_{t=1}^L \pi_{\mathrm{ref}}(a_t\mid x, s_{t-1}),
\end{equation*}
Here, on the right-hand side, with a slight abuse of notation, we use $\pi_{\text{ref}}$ to also denote a token-level reference policy
$\mathcal{X}\times\mathcal{S}\to\Delta(\mathcal{A})$, where $\mathcal{S}$ is the set of all token prefixes from steps 1 to $L$. Since beam search algorithms aggregate and process information at the token level, we focus throughout the remainder of the paper on the token-level policy $\pi_{\text{ref}}: \mathcal{X}\times\mathcal{S}\to\Delta(\mathcal{A})$. We assume black-box sampling access to the reference policy, without access to full logits or exact next-token probabilities, so token probabilities must be estimated from sampled continuations.


\paragraph{Goal and Evaluation.} Given a prompt $x$, the goal of parallel test-time computing is to design an algorithm $\texttt{Alg}$ that
interacts with the base policy $\pi_{\mathrm{ref}}(\cdot\mid s,x)$ through sampling and aggregation and outputs
a response $y\sim \pi_{\text{Alg}}(\cdot\mid x)$, aiming to maximize the expected ground-truth reward
$\mathbb{E}_{y\sim\text{Alg}(\cdot|x)}\big[r^\star(x,y)\big]$.

Let $y^\star:=\arg\max_{y\in\mathcal{Y}} r^\star(x,y)$ denote the optimal response that maximizes the \emph{true} outcome reward. We define the regret on instance $x$ as
\begin{equation}
\mathrm{Reg}(x)
= r^\star(x,y^\star)
-\mathbb{E}_{y\sim\pi_{\text{Alg}}(\cdot\mid x)}\!\left[r^\star(x,y)\right].
\end{equation}

\paragraph{Computational Complexity.} We measure the test-time computational cost of an algorithm $\texttt{Alg}$ by the number of \emph{token-level policy queries} it makes to the base model. Concretely, each time $\texttt{Alg}$ queries $\pi_{\mathrm{ref}}(\cdot\mid s,x)$ to sample from the next-token distribution at a prefix $s$, we count one query.
We use $N$ to denote the (per-prefix) sample complexity, i.e., the number of such queries issued at each prefix $s$,
and let $B(x)$ denote the total number of such queries on instance $x$. We refer to $B(x)$ as the total (query) budget of $\texttt{Alg}$.

\paragraph{Coverage Coefficient.} To quantify the ability of the base policy $\pi_{\mathrm{ref}}$ to generate the optimal response, 
following~\citet{huang2025best}, we introduce coverage coefficients that measure how difficult it is for the base policy $\pi_{\mathrm{ref}}$ to sample the optimal trajectory.
\begin{definition}[Token- and sequence-level coverage]\label{def:coverage}
Fix an instance $x$ and let $y^\star(x)=(a_1^\star,\dots,a_L^\star)$ be the unique optimal response, with optimal prefixes $s_{t-1}^\star=a^\star_{1:t-1}$.
Let $\pi^\star$ denote the oracle policy that deterministically follows the optimal token at each step, i.e.,
$\pi^\star(a_t^\star\mid s_{t-1}^\star,x)=1$ for all $t\in[L]$.
The token-level coverage coefficient is defined as
\begin{equation}
    C_t^\star(x)
    :=\frac{\pi^\star(a_t^\star\mid s_{t-1}^\star,x)}{\pi_{\mathrm{ref}}(a_t^\star\mid s_{t-1}^\star,x)}
    =\frac{1}{\pi_{\mathrm{ref}}(a_t^\star\mid s_{t-1}^\star,x)}.
\end{equation}
We further define $C^\star_{\max}(x):=\max_{t\in[L]} C_t^\star(x)$ and 
$\overline{C}(x):=\prod_{t=1}^{L} C_t^\star(x)$ as the sequence-level coverage coefficient.
\end{definition}

The quantity $C_t^\star(x)$ captures the inverse probability of sampling the optimal token at depth $t$ under $\pi_{\mathrm{ref}}$, while $\overline{C}(x)$ reflects the compounding difficulty along an autoregressive trajectory.
This notion is consistent with the $L_1$ coverage coefficient studied in \citet{huang2025best,di2025best}.



For notational convenience, we omit the dependence on the instance $x$ when no ambiguity arises. 

\section{Algorithm}\label{sec:algorithm}
\begin{wrapfigure}{r}{0.51\textwidth}
\vspace{-1.2em}
\begin{minipage}{\linewidth}
\small
\renewcommand{\baselinestretch}{0.9}\selectfont
\hrule height 0.8pt
\vspace{3pt}
\refstepcounter{algorithm}
\noindent\textbf{Algorithm~\thealgorithm}\ \ 
{\bf Vanilla-Beam} vs {\bf \textcolor{blue}{CF-Beam}}: 
CF-Beam adds highlighted lines {\bf \textcolor{blue}{(Line 14-15)}}
\label{alg:te-beam}
\vspace{3pt}
\hrule
\vspace{3pt}
\begin{algorithmic}[1]
\STATE \textbf{Input:} Base policy $\pi_{\mathrm{ref}}$, beam width $b \ge 1$,
       depth $L$, sample sizes $N$, threshold $\{\beta_t\}_{t=1}^{L}$,
       reward estimator $\hat r$, and prompt $x \in \mathcal{X}$.
\STATE $\mathcal{B}_0 \gets \{\emptyset\}$ \COMMENT{initial beam}
\STATE $Q(s_0) \gets 0$  \COMMENT{initial score over trajectory $s_0=\emptyset$}
\FOR{$t = 1$ \textbf{to} $L$}
  \STATE $\mathcal{D}_t \gets \emptyset$ \COMMENT{candidate prefixes at depth $t$}
  \FORALL{$s_{t-1} \in \mathcal{B}_{t-1}$}
    \STATE initialize $N_t(s_{t-1},a_t) \gets 0$ for all $a_t \in \mathcal{A}$
    \FOR{$i = 1$ \textbf{to} $N$}
      \STATE sample $a^{(i)}_t \sim \pi_{\mathrm{ref}}(\cdot \mid s_{t-1},x)$
      \STATE $N_t(s_{t-1},a^{(i)}_t) \gets N_t(s_{t-1},a^{(i)}_t) + 1$
    \ENDFOR
    \FORALL{$a_t \in \mathcal{A}$ \textbf{with} $N_t(s_{t-1},a_t) \ge 1$}
      \STATE $\hat\pi(a_t \mid s_{t-1},x) \gets N_t(s_{t-1},a_t) / N$
      \IF{\textcolor{blue}{$\hat\pi(a_t \mid s_{t-1},x) < \beta_t$}}
        \STATE \textbf{\textcolor{blue}{skip}} \hfill{\texttt{//CF-Beam only}}
      \ELSE
        \STATE $s'_t \gets s_{t-1} \oplus a_t$ \COMMENT{append $a_t$ to $s_{t-1}$}
        \STATE $Q(s'_t) \gets Q(s_{t-1}) \!+\! \log\hat\pi(a_t \mid s_{t-1},x)$
        \STATE $\mathcal{D}_t \gets \mathcal{D}_t \cup \{s_t'\}$
      \ENDIF
    \ENDFOR
  \ENDFOR
  \STATE $\mathcal{B}_t \gets \text{top-}b\{s'_t \in \mathcal{D}_t \text{ ranked by } Q(s'_t)\}$
\ENDFOR
\STATE $\hat y \gets \arg\max_{y \in \mathcal{B}_L} \hat r(x,y)$
\STATE \textbf{Output:} $\hat y$
\end{algorithmic}
\vspace{0pt}
\hrule height 0.6pt
\end{minipage}
\vspace{-2.0em}
\end{wrapfigure}
Beam search is a search-based decoding algorithm for LLM inference, which explores diverse responses and applies a step-level verifier to prune low-quality continuations at each step, aiming to ultimately concentrate on the most promising prefixes and reduce unnecessary test-time computational cost. 

\paragraph{Vanilla Beam Search (Vanilla-Beam).} Vanilla-Beam, as shown in \Cref{alg:te-beam} \citep{xie2023self, yu2025scaling}, generally maintains a beam of $b$ trajectories. At each step, it expands each trajectory by enumerating multiple next-token candidates (Lines 6-11), scores the newly formed trajectories under a chosen metric (Line 18), and prunes these trajectories to the top $b$ candidates, which forms the beam $\mathcal{B}$ for the next step (Line 23). After reaching the end-of-sequence or length limit, it selects the final output by maximizing an external reward model over the top-$b$ estimated log-likelihood candidates (Line 25). 

In particular, the log-likelihood is used as the scoring metric. 
For each time step $t\in[L]$, we draw $N$ fresh i.i.d.\ samples ${a}_{t}^{(1)},\ldots, {a}_{t}^{(N)}\sim \pi_{\text{ref}}(\cdot\mid s_{t-1},x)$, and then estimate the action distribution at step $t$ by the empirical frequencies:
$
\hat{\pi}(a_t|s_{t-1},x)=\frac{1}{N}\sum_{i=1}^{N}\textbf{1}({a}_{t}^{(i)}=a_t).
$
The score of a trajectory, denoted as $Q(s_t)$ is defined as the sum of estimated log-probabilities:
$
Q(s_t)=\sum_{i=1}^t \log\hat{\pi}(a_i|s_{i-1},x).
$
Samples at distinct prefixes use independent randomness. We keep only the top-$b$ trajectories for the next round, i.e.,
$
\mathcal{B}_t=\text{top-}b\left\{s_t\in{\mathcal{D}_t}:Q(s_t)\right\}.
$
If filtering leaves an empty beam, the algorithm returns any fixed response; ties are resolved by a fixed rule.



In practice, retaining only the top-$b$ trajectories can be fragile in Vanilla-Beam. Since the action space 
$\mathcal{A}$ is typically enormous, most tokens lie in the low-probability tail and are rarely sampled. Yet occasional tail samples amplified by finite-sample estimation and hard top-$b$ pruning can introduce spurious candidates and thereby \emph{crowd out} the optimal prefix $s^\star$.
Since beam search is path-dependent, eliminating $s^\star$ at an early step is irreversible and prevents recovery even if subsequent steps are estimated accurately. The following theorem formalizes this phenomenon by constructing a hard instance where Vanilla-Beam fails (i.e., the correct trajectory is discarded with a constant probability), without a sufficiently large sample size. Detailed proof of constructing the hard instance is deferred to Appendix \ref{Appendix:Counter example for not pruning beta}. 
\begin{theorem}
\label{th:vanilla_beam}
 There \textbf{exists an} LLM instance $\pi \in \Pi$ satisfying $50 \leq C_{i}^{\star}(x) := C^{\star}(x) \ll {(|\mathcal{A}|)^{1/4}}$ for all $i \in [L]$, such that if we set $N \leq (C^{\star}(x))^{2}/16$, the probability of eliminating the correct answer $y^\star$ using Vanilla-Beam with $2\leq b\leq N/2$ is at least $0.5$.
 \end{theorem}

\Cref{th:vanilla_beam} shows that without the sample size $N$ to be at least $N>\Omega((C_{\max}^{\star})^2)$, Vanilla-Beam cannot survive the correct response to the end of the process. 

\paragraph{Confidence-Filtered Beam Search (CF-Beam).} In order to reduce the sample complexity of Vanilla-Beam, inspired by \cite{di2025best}, we introduce a \emph{thresholding} step (Lines 14-15 in \Cref{alg:te-beam}), which filters low-probability samples based on the estimation of their probability of sampling. We call this modified algorithm Confidence-Filtered Beam Search (CF-Beam). To be specific, for each prefix $s_t,\ t\in [L]$, after generating $N$ actions, we keep only those satisfying $\hat{\pi}_t(a_t\mid s_{t-1},x)\geq\beta_t$, and then select the top-$b$ trajectories from such a filtered action set. We will show in \Cref{sec:regret} (\Cref{main theorem} and \Cref{cor:complexity}) that this simple modification yields a sample complexity bound of $\widetilde O(C_{\max}^{\star})$ for fixed horizon, prefix gap, and target accuracy, thereby improving the sample complexity from the quadratic to nearly linear dependence on $C^\star_t(x)$.

\section{Regret Analysis}

\subsection{Technical Assumptions}

In this subsection, we introduce the technical assumptions to establish the performance guarantee for the beam search method.

First, while the performance metric is defined with respect to the true reward $r^\star$, the algorithm typically has access only to an approximate reward model $\hat r$.
We capture the quality of $\hat r$ via a distributional estimation error under $\pi_{\mathrm{ref}}$, which is a commonly adopted assumption in the literature \citep{huang2025best, di2025best}. 
\begin{assumption}[Reward estimation error]\label{assump:reward-err}
For each instance $x \in \mathcal{X}$, the reward model satisfies
\begin{equation}
\mathbb{E}_{y \sim \pi_{\mathrm{ref}}(\cdot \mid x)}
\left[\left(r^\star(x, y)-\hat{r}(x, y)\right)^2\right]
\leq \epsilon_{\mathrm{RM}}^2(x).
\end{equation}
\end{assumption}

\begin{assumption}[Globally optimal response]\label{assump:opt}
For each instance $x$, there exists a unique optimal response
\begin{equation}
y^\star(x)=\arg\max_{y\in\mathcal{Y}} r^\star(x,y),
\end{equation}
with $r^\star(x,y^\star(x))=1$. Here,
$y^\star(x)=(a^\star_1,\dots,a^\star_L)$ captures the entire optimal path. Moreover, the reward model error at the optimum is bounded as
\begin{equation}
\big|r^\star(x,y^\star(x))-\hat{r}(x,y^\star(x))\big|
\le \epsilon_{\mathrm{opt}}(x).
\end{equation}
\end{assumption}

It is obvious that the reward model error is crucial for the success of inference time algorithms, and building an accurate reward model has been extensively studied in the Reinforcement Learning with Human Feedback (RLHF) framework \citep{wang2026reward,sun2025uncertainty,duan2025efficient,yu2025self}. Since our work mainly focuses on sampling strategy of the base model, we directly assume that we have access to a reward model with the estimation errors controlled by $\epsilon_{\text{RM}}^2$ and $\epsilon_{\text{opt}}$.

Furthermore, since the reward signal is only available at the \emph{sequence level}, the algorithm can evaluate a completed response $y$ using $\hat r(x,y)$, but does not observe intermediate-step rewards for partial prefixes.
In such settings, inference-time decisions must rely on signals intrinsic to the base policy, such as its self-likelihood and uncertainty over candidate continuations.
As having been justified by prior analyses \citep{cordero2025certified,feng2025optimal}, we impose the following prefix separability condition.
\begin{assumption}[Prefix competitiveness]\label{assump:prefix}
There exists a common $\kappa\in(0,1]$ such that, for each depth $t\in[L]$ and any prefix
$s=(a_1,\dots,a_t)\in\mathcal{A}^{t}$ with $s\neq s^\star_t=(a^\star_1,\dots,a^\star_t)$, we have
\begin{equation}
 \frac{1}{t}\sum_{i=1}^t\log\left(\frac{\pi_{\mathrm{ref}}(a^\star_i\mid s^\star_{i-1},x)}{\pi_{\mathrm{ref}}(a_i\mid s_{i-1},x)}\right) \ge \kappa,
\end{equation}
where $s_{i-1}=a_{1:i-1}$ and $s^\star_{i-1}=a^\star_{1:i-1}$.
\end{assumption}
Any larger positive gap lower bound also implies the condition with $\kappa=1$, so restricting the chosen lower bound to $(0,1]$ does not restrict the model class.

Such a gap condition is needed whenever intermediate pruning decisions
rely on likelihood-based signals rather than direct reward feedback. 
Related gap-dependent analyses have also appeared in self-consistency and
certification-based decoding settings; see, e.g.,
\citet{cordero2025certified, aeeneh2024new}. The following proposition
shows that, in our beam-search setting, this prefix-competitiveness
condition is not merely technical: without it, likelihood-based pruning
can irreversibly discard the optimal trajectory even with an arbitrarily
large sampling budget.

\begin{proposition}[Necessity of Assumption~\ref{assump:prefix}]
\label{prop:necessity}
Without Assumption~\ref{assump:prefix}, there exists an LLM instance
$(\pi_{\mathrm{ref}}, r^\star)$ on which, for every sample size $N\ge 1$,
beam search with beam width $b=1$ prunes the optimal trajectory $y^\star$
from the final beam with probability at least $3/4$.
\end{proposition}

The proof is deferred to Appendix~\ref{appendix: proof of proposition 1},
and additional discussion is provided in
Appendix~\ref{appendix: dis assump3}.

\subsection{Regret Upper Bound}\label{sec:regret}
 
The regret arises from three factors: (i) the accuracy of the reward model used for the final selection, (ii) the capacity of the base model $\pi_{\mathrm{ref}}$, and (iii) search failures, i.e., events under which the optimal response is eliminated during inference. The first factor is captured by Assumptions~\ref{assump:reward-err} and~\ref{assump:opt}, and the second by the coverage coefficients in Definition~\ref{def:coverage}. To control the third factor, we define good events that separate two failure modes: the optimal action fails to pass confidence filtering, or the optimal prefix passes filtering but is discarded by top-$b$ pruning.
 
\paragraph{Pre-sampled histograms.}
Since the algorithm samples only at prefixes in the current beam, the empirical frequency $\hat\pi(a_t^\star\mid s_{t-1}^\star,x)$ is undefined once an optimal prefix has been discarded. To define it for every $t$, we use an equivalent description of the sampling process. Before the algorithm runs, we independently draw, for every prefix $s$ of depth less than $L$, a histogram of $N$ i.i.d.\ samples from $\pi_{\mathrm{ref}}(\cdot\mid s,x)$; the histogram of $s$ is revealed only when the algorithm expands $s$. Since each prefix is expanded at most once, the output of the algorithm has the same distribution as under fresh sampling, and the construction is used only in the analysis. Under this coupling, the empirical frequencies along the optimal path are defined whether or not the path survives, and histograms at distinct prefixes are mutually independent.
 
\begin{definition}[Good events]\label{def: good event}
We define the following events.
 
\textbf{(i) Filtering event.}
The optimal action passes confidence filtering at every depth:
\[
\mathcal{E}_{\beta}:=\bigl\{\forall t\in[L]:\ \hat\pi(a_t^\star\mid s_{t-1}^\star,x)\ge\beta_t\bigr\}.
\]
 
\textbf{(ii) Beam-survival event.}
Let $\mathrm{rank}_t(s)$ denote the rank of $s\in\mathcal D_t$ when $\mathcal D_t$ is sorted by $Q(\cdot)$ in descending order, with ties broken by the rule of the algorithm. Define
\[
\mathcal{E}_b:=\bigl\{\forall t\in[L]:\ s_t^\star\in\mathcal D_t
\ \Rightarrow\ \mathrm{rank}_t(s_t^\star)\le b\bigr\}.
\]
The overall good event is $\mathcal E:=\mathcal E_\beta\cap\mathcal E_b$.
\end{definition}
 
By induction over the depths, $\mathcal E$ holds if and only if $y^\star\in\mathcal B_L$: on $\mathcal E_\beta$, every optimal prefix in the beam generates its optimal extension as a candidate, and on $\mathcal E_b$, top-$b$ pruning retains this extension. We refer to $\mathbb P(\mathcal E^c)$ as the \emph{search error}.  

Now we present our main result on the regret of the CF-Beam algorithm.
 
\begin{theorem}[Regret of CF-Beam]\label{main theorem}
Suppose $L\ge2$, $N\ge \frac{48C_{\max}^{\star}}{\kappa^2}\max\{2,\log(2C_{\max}^{\star})\}$, and $\beta_t=(1-1/L)/C_t^\star(x)$ for all $t\in[L]$. Under Assumptions~\ref{assump:reward-err}, \ref{assump:opt}, and~\ref{assump:prefix}, the regret of CF-Beam in \Cref{alg:te-beam} satisfies
\begin{equation}
\mathrm{Reg}(x)\le\mathbb P(\mathcal E^c)
+\underbrace{\epsilon_{\mathrm{opt}}(x)+2\sqrt{\overline C(x)\epsilon_{\mathrm{RM}}^2(x)}}_{\text{reward-model error}},
\label{eq:main-regret}
\end{equation}
where the search error satisfies
\begin{equation}
\mathbb P(\mathcal E^c)\le
\underbrace{L\exp\!\left(-\frac{N}{2C_{\max}^{\star}L^2}\right)}_{\text{filtering error}}
+\underbrace{\frac{2C_{\max}^{\star}}{b}\exp\!\left(-\frac{\kappa^2N}{48C_{\max}^{\star}}\right)}_{\text{pruning error}}.
\label{eq:main-search-error}
\end{equation}
\end{theorem}
 Theorem~\ref{main theorem} separates the regret into a search error and a reward-model error. The search error accounts for filtering out an optimal action or pruning its prefix from the beam. Both probabilities decay exponentially with the per-prefix sample size $N$; their rates depend on the token-level coverage $C_{\max}^\star$, the horizon $L$ for filtering, and the prefix gap $\kappa$ for pruning. A wider beam reduces the pruning term, since discarding the optimal prefix requires at least $b$ filtered competitors to match or exceed its score. Once the optimal response survives, the remaining regret comes from the final reward-based selection. Its bound involves the sequence-level coefficient $\overline C(x)$, which arises when transferring the reward-model error guarantee under $\pi_{\mathrm{ref}}$ to the selected output; this term does not decrease with $N$ or $b$. The threshold $\beta_t$ depends on the unknown $C_t^\star(x)$, so this result gives an oracle guarantee rather than a directly implementable rule. For practical use, Appendix~\ref{appendix: details of numerical experiments} describes an empirical variant of CF-Beam that replaces the unknown optimal-token probability $1/C_t^\star(x)$ with the largest empirical next-token frequency at the current prefix, and sets the threshold to $\beta_t(s)=\gamma\max_{a\in\mathcal A}\hat\pi(a\mid s,x)$ for a constant $\gamma\in(0,1)$. This proxy is reasonable when the optimal continuation is locally competitive, in which case the optimal-token probability is of the same order as the largest next-token probability.

\paragraph{Proof sketch.}
We first bound the probability that the optimal response is lost during search. At each depth, the empirical count of the optimal action is binomial; a Chernoff lower-tail bound and a union bound over depths control the filtering error. The threshold lies a relative $1/L$ below the mean frequency, which yields the $L^2$ factor in this bound. For pruning, we count the filtered competitors whose empirical scores match or exceed that of the optimal prefix. Truncated binomial moment bounds control their score differences even when low-probability actions happen to pass the filter. Assumption~\ref{assump:prefix} supplies the true likelihood gap, and summing over competitors uses their total probability mass rather than their number. If the optimal prefix passes filtering but is pruned, at least $b$ competitors must match or exceed its score; Markov's inequality then controls the pruning error.

On $\mathcal E$, the optimal response reaches the final beam, so the remaining regret comes from reward-based selection. For each response $y$, the joint probability $\Pr(\hat y=y,\mathcal E)$ is at most $\pi_{\mathrm{ref}}(y\mid x)\prod_{t=1}^L\beta_t^{-1}\le 4\overline C(x)\pi_{\mathrm{ref}}(y\mid x)$. This comparison and Cauchy--Schwarz transfer the reward-model error bound under $\pi_{\mathrm{ref}}$ to the selected response, while the pointwise error at $y^\star$ contributes $\epsilon_{\mathrm{opt}}(x)$. We provide the comprehensive proof in Appendix~\ref{Appendix: proof of main theorem}.
 
\begin{corollary}[Sample complexity]\label{cor:complexity}
Under the assumptions and threshold choice of Theorem~\ref{main theorem}, for $\delta\in(0,1)$, it suffices to take
\begin{equation}
N\gtrsim C_{\max}^{\star}\max\left\{
L^2\log\frac{2L}{\delta},\quad
\frac{1}{\kappa^2}\left(\log(eC_{\max}^{\star})+
\left[\log\frac{1}{b\delta}\right]_+\right)
\right\}
\label{eq:sample_complexity_bound}
\end{equation}
to achieve regret at most $\delta+\epsilon_{\mathrm{opt}}(x)+2\sqrt{\overline C(x)\epsilon_{\mathrm{RM}}^2(x)}$, where $[u]_+:=\max\{u,0\}$ and the implicit constant is universal.
\end{corollary}
 
\Cref{cor:complexity} shows that the sufficient per-prefix sample size grows linearly in $C_{\max}^\star$ and quadratically in $L$, up to logarithmic factors, and only logarithmically in $1/\delta$. Since the total query budget is at most $LbN$, the total budget is $\widetilde O(bL^3C_{\max}^\star)$ for fixed $\kappa$ and $\delta$. The beam width enters only through $\log(1/(b\delta))$: a wider beam relaxes the requirement from the pruning error until the filtering error or the baseline condition dominates.
 
\paragraph{Benefit of confidence filtering.}
For fixed $L$, $\kappa$, and $\delta$, \Cref{cor:complexity} shows that $N=\widetilde O(C_{\max}^\star)$ samples per prefix suffice for the search error to be at most $\delta$. In contrast, \Cref{th:vanilla_beam} exhibits an instance satisfying Assumption~\ref{assump:prefix} on which Vanilla-Beam with $2\le b\le N/2$ and $N\le (C^\star)^2/16$ discards $y^\star$ with probability at least $1/2$. Confidence filtering therefore reduces the per-prefix sample requirement from quadratic to nearly linear in the token-level coverage.
 
 

\subsection{Comparison with Other Test-time Methods}\label{sec:comparison}

In this subsection, we compare the search-induced component of CF-Beam with other commonly adopted test-time algorithms for which regret guarantees have recently been established.

First, \citet{di2025best} proposed the \textit{Best-of-Majority} algorithm, which first applies majority voting (self-consistency) to discard low-frequency candidates and then selects the final response from the remaining set using an external reward model. They established a regret bound given by
$\textstyle
\mathrm{Reg}(x)\lesssim \epsilon_{\mathrm{opt}}(x)
+ \sqrt{\overline{C}(x)\,\epsilon_{\mathrm{RM}}^{2}(x)}
+ \overline{C}(x)\exp\!\left(-{\widetilde{N}}/{\overline{C}(x)}\right),$
where $\overline{C}(x):=\prod_{t=1}^{L} C_t^\star(x)$ captures the coverage along the entire trajectory in our path-based notation, and $\widetilde{N}$ refers to the number of complete responses sampled. To make the search-induced term at most $\delta$, the sample size $\widetilde{N}$ needs to satisfy
$\widetilde{N} \gtrsim \overline{C}(x)\log\!\left({\overline{C}(x)}/{\delta}\right)$, which leads to the total computation budget $B=L\widetilde{N}$ on the order of
$B \gtrsim L\,\overline{C}(x)\log\!\left({\overline{C}(x)}/{\delta}\right)$.
Whenever $C^\star_{\text{min}}:=\min_{1\leq t\leq L} C_t^\star(x) >1$, this budget can scale {\bf exponentially} with the horizon length $L$.
This exponential scaling of the coverage-dependent term also appears in other test-time algorithms such as \textit{Best-of-$N$} \cite{huang2025best}, where the regret bound takes the form of $O\left({\overline{C}(x)\cdot\log(1/\varepsilon_{\mathrm{RM}}(x))}/{\widetilde{N}}+\sqrt{\widetilde{N}\cdot\varepsilon_{\mathrm{RM}}^2(x)}\right).$ Similarly, the dependence on $\overline{C}(x)$ can lead to exponentially large sample complexity as $L$ increases.

In contrast, our CF-Beam method controls the rare-failure term with a budget
$B = L b N$ that satisfies Corollary \ref{cor:complexity}, which scales only {\bf polynomially} in $L$
(up to the beam-width factor $b$).
Hence, beam-pruning combined with $\beta$-filtering effectively replaces the potentially exponential dependence
on $\overline{C}(x)$ with a polynomial dependence on $L$ in the search-induced component, yielding substantially reduced computation-cost scaling in long-horizon settings when the reward-model error is controlled.


The above comparison suggests that CF-Beam can be more favorable on hard instances. Compared to sequence-level parallel sampling methods, a key advantage of CF-Beam is that its search-induced guarantee scales with $L C_{\max}^{\star}$ rather than the product $\overline{C}(x)=\prod_{t=1}^L C_t^\star(x)$. 
When all $C_t^\star(x)$ are close to $1$ and $L$ is moderate, sequence-level methods such as Best-of-$N$ or Best-of-Majority may be more sample-efficient. 
In contrast, when some $C_t^\star(x)$ are large, $\overline{C}(x)$ grows rapidly, and CF-Beam becomes more favorable due to its dependence on $L C_{\max}^{\star}$ in the search-induced term.
This qualitative trend is consistent with the empirical scaling study of \citet{snell2024scaling}.

\subsection{Regret Lower Bound}
In this subsection, we establish the lower bound for the beam search algorithm. 
\begin{theorem}\label{regret lower bound}
For the CF-Beam algorithm, when $b\geq 2$, there exist hard instances with a sufficiently small constant $\kappa^{\prime}>0$ and $C^\star_t(x)={|\mathcal{A}|}/({1+|\mathcal{A}|\kappa^{\prime}})$ for all $t\in[L]$, such that the regret can be lower bounded as $\Omega\left(\sqrt{{\overline{C}(x)\epsilon_{\text{RM}}^{2}(x)}}\right).$
\end{theorem}
This lower bound matches the reward model error term in the upper bound in Theorem~\ref{main theorem}, demonstrating the minimax optimality of the regret with respect to the reward estimation error $\epsilon_{\text{RM}}^{2}(x)$. To be specific, the lower bound is determined by (i) the difficulty of the problem $\overline{C}$ and (ii) the accuracy of the reward model. These two dependencies are necessary, since outputting the optimal response requires both the model's ability to generate it and the reward model's ability to identify it.

\citet{di2025best} established a general lower bound for the pass@k inference problem with a similar structure.
However, their construction cannot be directly applied to our setting, as it does not satisfy
Assumption~\ref{assump:prefix}. Our lower bound instead constructs a family of instances
with an arbitrarily small prefix gap $\kappa$, under which it is information-theoretically difficult for any
intermediate pruning strategy to eliminate all suboptimal prefixes with high probability before the final step.
The incurred regret is then driven by the inability of the imperfect reward model to always select the true optimum among the surviving candidates.

\subsection{Extension to Reward Model Free Setting}
\begin{wrapfigure}{r}{0.5\textwidth}
\begin{minipage}{\linewidth}
\small
\renewcommand{\baselinestretch}{0.9}\selectfont

\hrule height 0.8pt
\vspace{3pt}
\refstepcounter{algorithm}
\noindent\textbf{Algorithm~\thealgorithm}\ \ Self-consistent CF-Beam \par
\label{alg:mv-beam}
\vspace{3pt}
\hrule
\vspace{3pt}
\begin{algorithmic}[1]
\STATE \textbf{Given:} base policy $\pi_{\mathrm{ref}}$, beam width $b \ge 1$,
       depth $L$, sample sizes $N$, threshold $\{\beta_t\}_{t=1}^{L}$.
\STATE \textbf{Input:} prompt $x \in \mathcal{X}$.
\STATE Execute Algorithm~\ref{alg:te-beam}, Lines 2--24.
\STATE $\hat y \gets \arg\max_{y \in \mathcal{B}_L} Q_L(y)$.
\STATE \textbf{return} $\hat y$
\end{algorithmic}
\vspace{3pt}
\hrule height 0.8pt
\end{minipage}
\vspace{-1.2em}
\end{wrapfigure}

In the case when a reliable reward model is unavailable (or severely noisy), finding the most probable path becomes a natural metric to identify the optimality of the path. 
To this end, we can derive a reward-model-free variant of CF-Beam by
replacing the final selection rule, namely, instead of selecting the
candidate with the largest estimated reward, we define
\(Q_L(y):=\sum_{t=1}^L\log\hat{\pi}(a_t\mid s_{t-1},x)\) and output
$
\hat{y}\leftarrow\arg\max_{y\in \mathcal{B}_L}Q_L(y),
$
i.e., the most ``self-consistent'' candidate within the final beam according
to the intrinsic score \(Q_L\).
This modification yields Algorithm~\ref{alg:mv-beam}, which can be viewed as a reward-model-free variant of CF-Beam in Algorithm~\ref{alg:te-beam}.

The same filtered-score analysis controls this variant, with an additional comparison at the final step: the selected complete response must have the largest empirical likelihood, rather than merely remain in the beam. The proof is deferred to Appendix~\ref{Proof of Corollary {corollary: bs-mv alg}}.

\begin{corollary}\label{corollary: bs-mv alg}
Suppose $y^\star$ is an optimal response with $r^\star(x,y^\star)=1$, Assumption~\ref{assump:prefix} holds, $N\ge\frac{48C_{\max}^{\star}}{\kappa^2}\max\{2,\log(2C_{\max}^{\star})\}$, and $\beta_t=1/(2C_t^\star(x))$. The regret of \Cref{alg:mv-beam} satisfies
\begin{align}
\mathrm{Reg}(x)
&\le L\exp\!\left(-\frac{N}{8C_{\max}^{\star}}\right)
+\frac{2C_{\max}^{\star}}{b}\exp\!\left(-\frac{\kappa^2N}{48C_{\max}^{\star}}\right)\nonumber+\left[C_{\max}^{\star}\exp\!\left(-\frac{\kappa^2N}{48C_{\max}^{\star}}\right)\right]^L.
\label{eq:reward-free-regret}
\end{align}
In particular, $N\gtrsim(C_{\max}^{\star}/\kappa^2)\log(eLC_{\max}^{\star}/\delta)$ suffices to achieve regret at most $\delta\in(0,1)$. No reward-model accuracy assumption is needed.
\end{corollary}

The first two terms account for filtering and intermediate beam pruning, and the last controls final selection by the empirical score. Since the guarantee does not involve reward-model error, the threshold multiplier can be fixed at $1/2$ independently of $L$. This removes the $L^2$ factor from the filtering contribution to the sufficient per-prefix sample size in Corollary~\ref{cor:complexity}. The resulting total query budget is $\widetilde O(LbC_{\max}^{\star}/\kappa^2)$, with only logarithmic dependence on $L$ in the per-prefix sample requirement.

\section{Experiments}\label{Section: Experiment}

\begin{figure*}[t]
\centering

\begin{minipage}[t]{0.72\textwidth}
    \centering
    \begin{subfigure}[t]{0.32\linewidth}
        \centering
        \includegraphics[width=\linewidth]{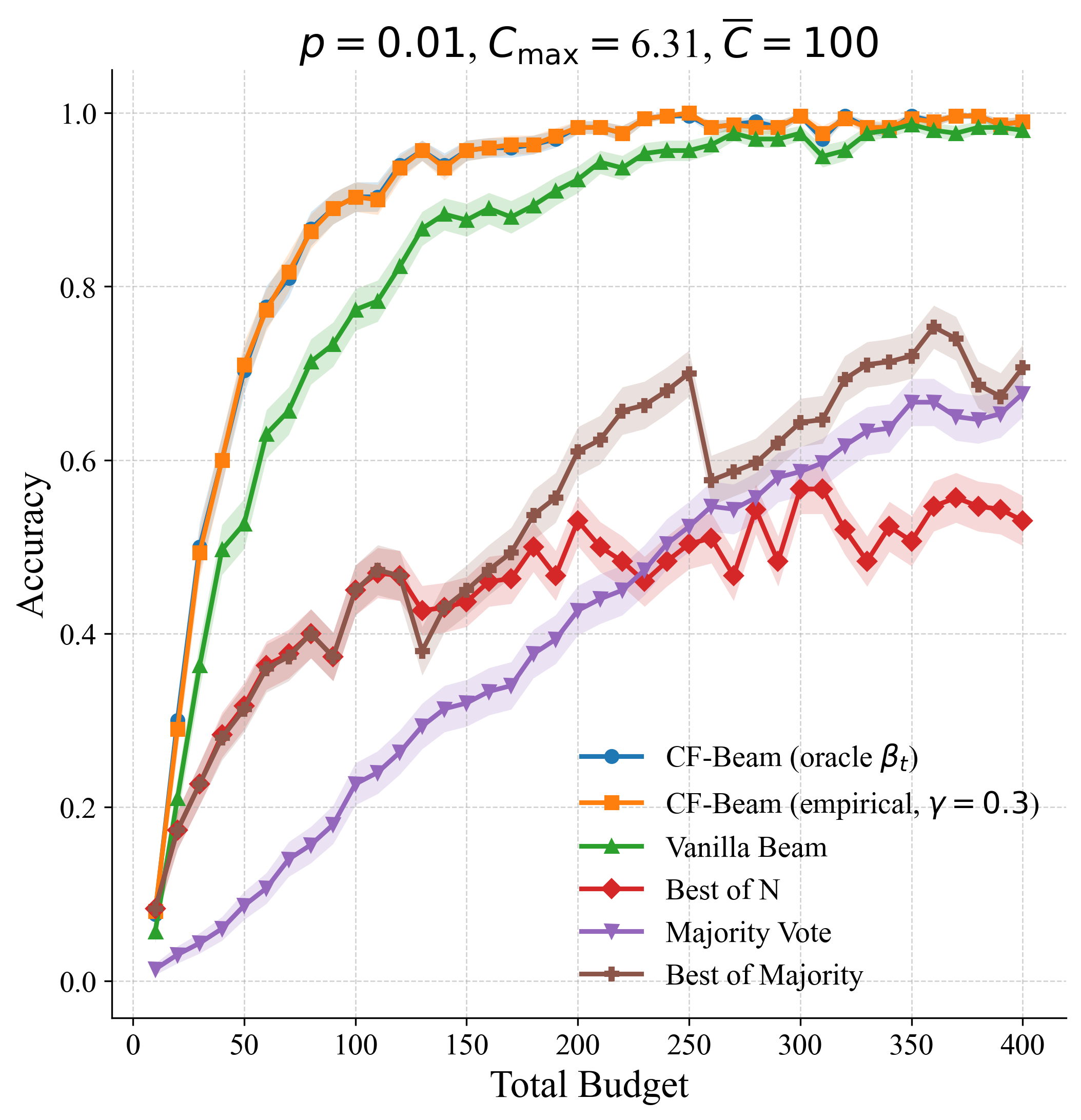}
        \caption{$p=0.01$}
        \label{fig:diff_a}
    \end{subfigure}\hfill
    \begin{subfigure}[t]{0.32\linewidth}
        \centering
        \includegraphics[width=\linewidth]{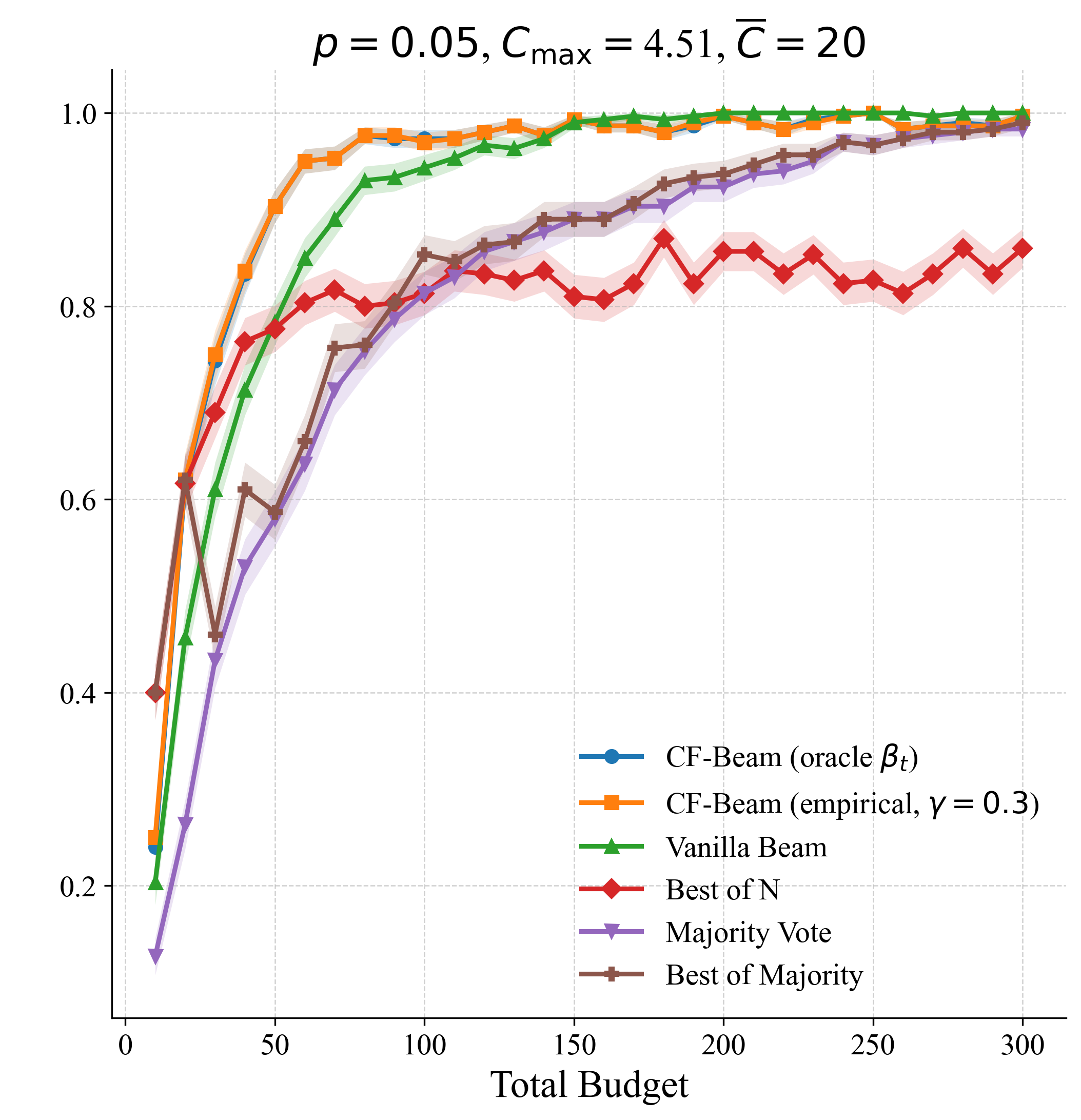}
        \caption{$p=0.05$}
        \label{fig:diff_b}
    \end{subfigure}\hfill
    \begin{subfigure}[t]{0.32\linewidth}
        \centering
        \includegraphics[width=\linewidth]{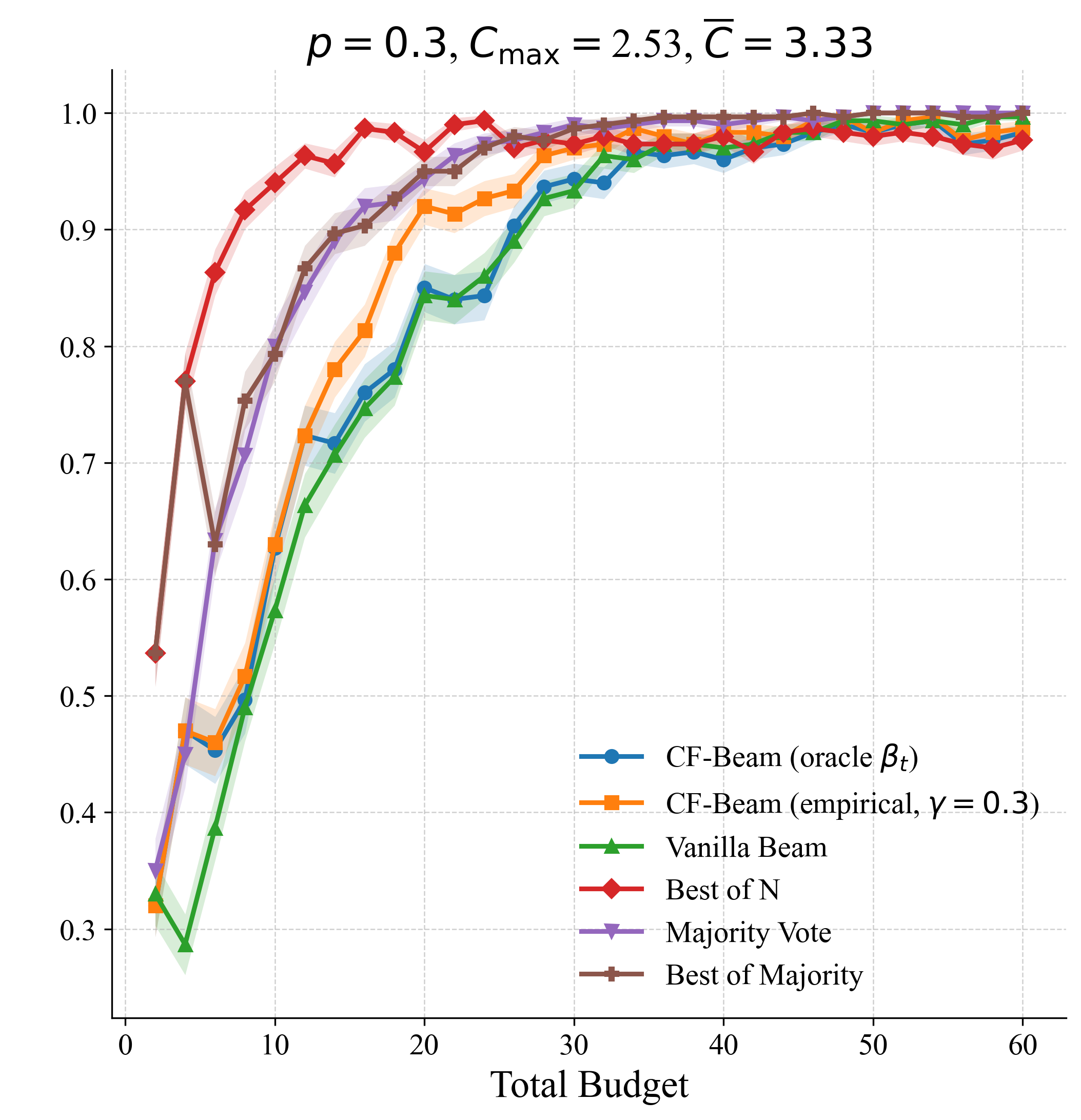}
        \caption{$p=0.3$}
        \label{fig:diff_c}
    \end{subfigure}

    \caption{Algorithm performance on different difficulty-level problems.}
    \label{fig:diff_levels}
\end{minipage}\hfill
\begin{minipage}[t]{0.24\textwidth}
    \centering
    \includegraphics[width=\linewidth]{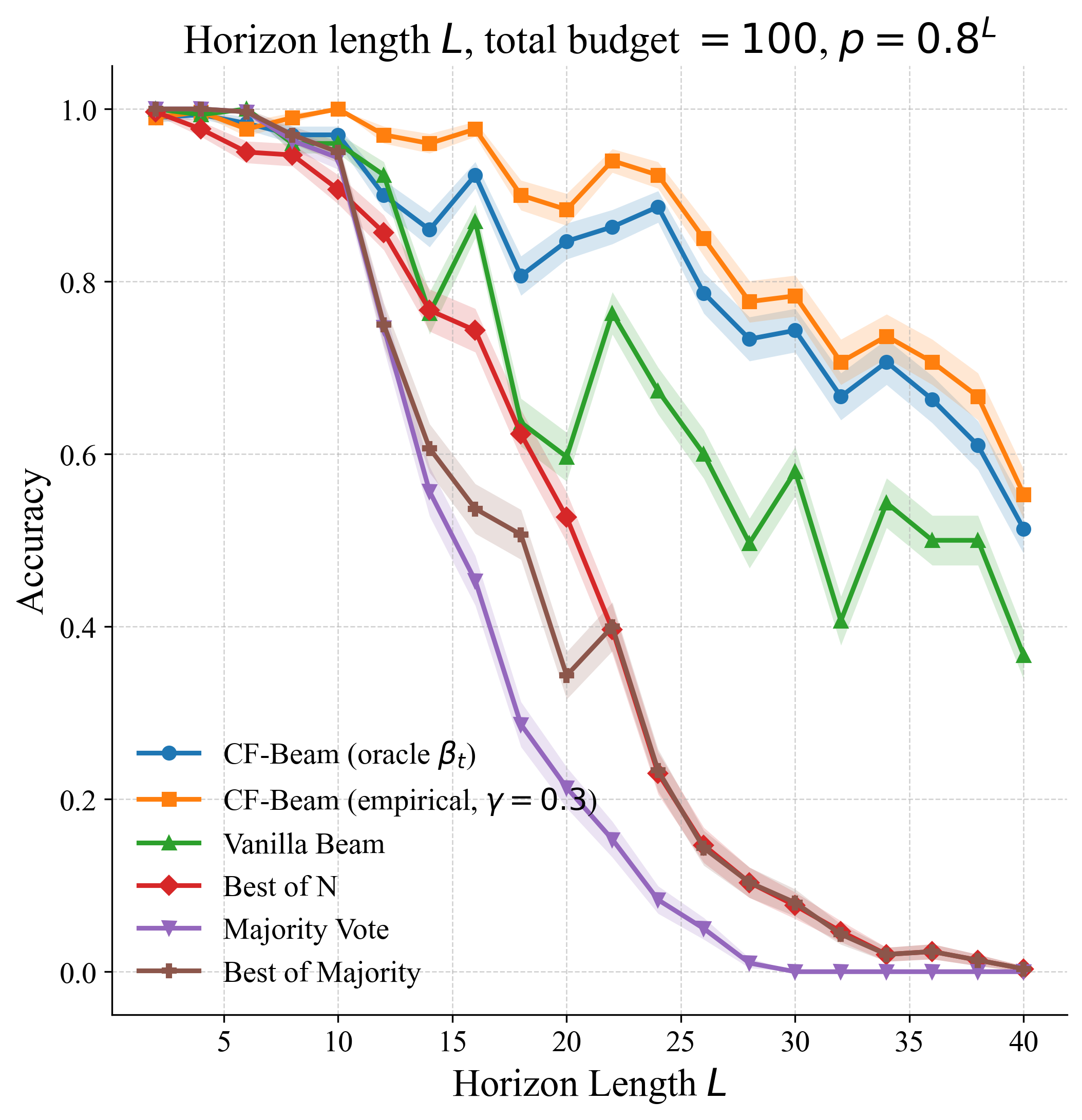}
    \caption{Algorithm performance vs. horizon length.}
    \label{fig:horizon_len}
\end{minipage}

\end{figure*}

We validate our theoretical findings by comparing two implementations of CF-Beam,
Oracle CF-Beam and Empirical CF-Beam, against four baselines: Vanilla beam search, Best-of-$N$ \citep{huang2025best}, Majority Voting, and Best-of-Majority \citep{di2025best}. Oracle CF-Beam uses the threshold in Theorem~\ref{main theorem}, while Empirical CF-Beam sets $\beta_t(s)=\gamma\max_{a\in\mathcal A}\hat\pi(a\mid s,x)$ with $\gamma=0.3$, without requiring the optimal-token probability. Following the RLVR setting, we evaluate each algorithm on a synthetic LLM simulator paired with a noisy reward model of error level $\epsilon_{\mathrm{RM}}^2 = 0.01$, and report the accuracy of identifying the optimal response. Baseline descriptions, the experimental protocol, simulator construction, and parameter setups are deferred to Appendix~\ref{appendix: numerical experiments}. We also evaluate Empirical CF-Beam on a pretrained LLM in Appendix~\ref{appendix:llm-experiments}.

Figure~\ref{fig:diff_levels} compares the algorithms under varying total budgets across three difficulty levels, parameterized by the per-rollout success probability $p \in \{0.01, 0.05, 0.3\}$ (larger $p$ = easier). CF-Beam algorithms achieve the largest advantage in the hardest setting ($p=0.01$, Figure~\ref{fig:diff_a}), where sparse coverage and tail fluctuations are most pronounced---precisely the regime that amplifies the failure mode of Theorem~\ref{th:vanilla_beam}. As $p$ increases (Figures~\ref{fig:diff_b}--\ref{fig:diff_c}), the gap narrows and Vanilla-Beam catches up under sufficiently large budgets. This trend mirrors the relationship between $\overline{C}$ and $C^\star_{\max}$ shown in each panel: as $p$ grows, $\overline{C}$ approaches $C^\star_{\max}$, reducing CF-Beam's edge. These observations align with our theoretical discussion in \Cref{sec:comparison}.

Figure~\ref{fig:horizon_len} reports accuracy as the horizon $L$ grows, with per-step success probability $0.8$. Sequence-level baselines (Best-of-$N$, Best-of-Majority, Majority Voting) degrade rapidly, while beam-based methods remain robust, and CF-Beam consistently attains the highest accuracy. This is consistent with our theory: CF-Beam's dominant sample-complexity term scales only polynomially in $L$, whereas sequence-level methods can suffer exponential dependence on $L$.

\section{Related Works}
In this section, we focus on related work in theoretical studies of test-time scaling. Additional related work is deferred to Appendix~\ref{Appendix: related works}.

Despite extensive literature on test-time computing \citep{wang2022self,kang2025scalable,jinnai2024regularized,ichihara2025evaluation,wei2022chain,ghosal2025does,liu2025trust}, especially beam search \citep{xu2025collaborative,xie2023self,zhu2024deductive}, theoretical analyses have only emerged recently. \citet{hu2024unveiling} decomposed the statistical error of Chain-of-Thought (CoT) estimators and showed it vanishes with more demonstrations. \citet{huang2024self} analyzed \textit{Best-of-N} and showed that larger $N$ does not necessarily improve performance. \citet{cordero2025certified} and \citet{aeeneh2024new} established convergence rates for Majority Voting, while \cite{di2025best} proposed \emph{Best-of-Majority}, which selects the top-$k$ response from a set of highly self-consistent candidates with non-degrading sample complexity. \citet{yu2025limits} combined Best-of-N with a rewind-and-repeat strategy \citep{kim2024guaranteed,beirami2024theoretical} to derive stronger guarantees.

The existing theoretical studies on {\em beam search} methods remains rather limited.  \citet{DBLP:conf/emnlp/MeisterCV20} characterized the implicit objective optimized by beam search, and \cite{fadeeva2025don} analyzed beam search for self-consistency-based uncertainty quantification. To our knowledge, no prior work studies the test-time scaling guarantees of beam search;  Our work fills this research gap and provides the first test-time computing guarantee for beam search type algorithms.

\section{Conclusion and Limitations}
This paper studied test-time guarantees for beam search in LLM reasoning,
where model likelihood scores intermediate trajectories and an external
reward model is queried only after full responses are generated. We showed
that vanilla beam search can be theoretically suboptimal, and proposed
\emph{Confidence-Filtered Beam Search} to filter low
self-consistency tokens before beam selection. CF-Beam improves the
search-induced scaling over sequence-level methods on hard and
long-horizon instances. The limitation is that this advantage relies on
prefix alignment between likelihood and final reward; otherwise,
reward-optimal prefixes may be pruned before final reward evaluation.
Future work includes prompt-adaptive choices between beam-based and
sequence-level methods, and settings with external rewards or process
verifiers available before full responses are generated.

\medskip
{\small
\bibliographystyle{apalike}
\bibliography{ref}
}





\newpage
\appendix

\allowdisplaybreaks
\begin{table}[H]
\centering 
\caption{Summary of Notations Used in the Paper} 
\label{tab:notations} 
\begin{tabularx}{\linewidth}{@{} l >{\raggedright\arraybackslash}X @{}}
\toprule 
\textbf{Notation} & \textbf{Description} \\ 
\midrule
$\pi_{\text{ref}}(\cdot|x)$ & a fixed LLM model that maps ${\mathcal{X}\to\Delta(\mathcal{A}^L)}$\\
$L$ & Response length (horizon) of $y=(a_1,\ldots,a_L)$. \\
$a$ & An action/token from the vocabulary $\mathcal{A}$. \\
$s_t$ & Prefix (trajectory) up to step $t$: $s_t=(a_1,\ldots,a_t)$. \\
$b$ & Beam width, i.e., the number of trajectories retained at each step. \\
$B$ & Test-time compute budget, i.e., the total number of samples drawn from $\pi_{\mathrm{ref}}$. \\
$N$ & Number of samples drawn per prefix $s_t \in \mathcal{B}_t$ at each step. \\
$\mathcal{D}_t$ & Candidate set after confidence filtering at step $t$. \\
$\mathcal{B}_t$ & Beam set at step $t$, consisting of at most $b$ selected prefixes (i.e., $|\mathcal{B}_t|\le b$). \\
$Q(s_t)$ & Estimated score of prefix $s_t$, e.g., $Q(s_t)=\sum_{i=1}^{t}\log\hat{\pi}(a_i\mid s_{i-1},x)$. \\
$\kappa$ & Prefix gap parameter between optimal and suboptimal trajectories; see Assumption~\ref{assump:prefix}. \\
$\epsilon_{\mathrm{RM}}^{2}$ & Reward-model estimation error; see Assumption~\ref{assump:reward-err}. \\
$\epsilon_{\mathrm{opt}}$ & Reward-model error on the optimal response; see Assumption~\ref{assump:opt}. \\
$C_t^\star(x)$, $C_{\max}^\star(x)$, $\overline{C}(x)$ & Coverage coefficients; see Definition~\ref{def:coverage}. \\
$\beta_t$ & Confidence-filtering threshold at step $t$. \\
\bottomrule
\end{tabularx}
\end{table}

\textbf{Notations.} 
We use $\mathrm{Bin}(N,p)$ to denote the binomial distribution with parameters $N \in \mathbb{N}^+$ and $p \in [0,1]$.
\section{Related Works}\label{Appendix: related works}
\paragraph{Test-time computing.} Beyond post-training scaling laws, recent work has increasingly focused on scaling computation at inference time to improve the quality of a fixed base model. Existing work mainly falls into two approaches: parallel and sequential test-time computation. For parallel methods, the compute budget is used to generate multiple candidate trajectories simultaneously, increasing the probability of finding a high-quality response. Existing such type of methods include Best-of-N (BoN) \citep{huang2025best, jinnai2024regularized, ichihara2025evaluation}, majority voting \citep{wang2022self,kang2025scalable}, and beam search \citep{xie2023self, yu2025scaling}. In sequential test-time computation, computation is allocated adaptively through iterative refinement or search, using methods like Chain-of-Thought \citep{wei2022chain, ghosal2025does, hu2024unveiling} and self-verification \citep{liu2025trust,lee2025revise}. Some recent studies aimed to bridge the gap between existing test-time computation methods. \citet{di2025best} proposes Best-of-Majority, which used majority voting as a filtering mechanism and an external reward model to select the final answer. \citet{yu2025limits} included a sequential technique for BoN sampling. The beam search method we study here mainly falls into the first category, since candidates are generated in parallel rather than sequentially.

\paragraph{Beam search.}
Beam search algorithms have been widely used as a decoding strategy for large language models. \citet{xu2025collaborative} and \citet{xie2023self} mainly used self-consistency as the evaluation metric for top-$b$ selection; \citet{zhu2024deductive} introduced beam search to increase the successful rate of chain of thought; \citet{yu2025scaling} empirically compared beam search algorithms with the methods like BoN and majority voting, and pointed out the weakness and strengths of beam search-based algorithms. Despite the empirical success of beam search and its many variants, there has been rather limited theoretical studies of beam search algorithms. \citet{DBLP:conf/emnlp/MeisterCV20} studied what implicit objective function beam search optimizes. A recent work \citet{fadeeva2025don} studied beam search for consistency-based uncertainty quantification and provided a theoretical comparison between beam-based candidate generation and multinomial sampling estimators. To our best knowledge, there has not been any work theoretically analyzing the test-time scaling performance of beam search. Our work provides the first test-time computing guarantee for beam-search decoding algorithms in modern LLM settings.

\paragraph{Self-consistency.}Self-consistency improves reasoning by sampling multiple independent solution paths and selecting the answer that appears most consistently across them. Existing works have demonstrated the effectiveness of self-consistency, when it is used in self-evaluated post-training \citep{fadeeva2025don, liu2025trust,zhao2025learning,sun2025theoretical} and test-time computing \citep{aggarwal2023let, chen2023universal}. \citet{cordero2025certified, feng2025optimal} and \citet{aeeneh2024new} provided theoretical analysis of the convergence rate of self-consistency methods. Our CF-Beam is related to self-consistency methods in that it uses a self-consistency score to prune beams at intermediate steps.

\section{Discussion of Assumptions and Computational Costs} \label{appendix: dis assump3}

We clarify how response uniqueness and prefix competitiveness enter our guarantees, and distinguish the sampling budget from the cost of maintaining and processing candidate prefixes.

\subsection{Uniqueness of the Optimal Response}
Assumption~\ref{assump:opt} makes the optimal response unique to simplify the path-based presentation. The analysis can accommodate several reward-optimal responses by tracking one that satisfies the required prefix and reward-model error conditions. As seen in the following remark.

\begin{remark}[Multiple reasoning trajectories leading to the same answer]
\label{rem:multiple-trajectories}
Multiple reasoning trajectories yielding the same correct answer are a special case of nonunique reward-optimal responses. Let $\mathcal Y^\star(x):=\{y\in\mathcal A^L:r^\star(x,y)=1\}$ be nonempty, and fix an optimal witness $y^\dagger=(a_1^\dagger,\ldots,a_L^\dagger)\in\mathcal Y^\star(x)$ with prefixes $s_t^\dagger:=a_{1:t}^\dagger$. Suppose Assumption~\ref{assump:prefix} holds with $s_t^\star$ replaced by $s_t^\dagger$, and $|r^\star(x,y^\dagger)-\hat r(x,y^\dagger)|\le\epsilon_{\mathrm{opt}}(x)$. The prefix gap must hold against every distinct competing prefix, including prefixes of other correct trajectories.

Define $C_t^\dagger(x):=1/\pi_{\mathrm{ref}}(a_t^\dagger\mid s_{t-1}^\dagger,x)$, $C_{\max}^\dagger(x):=\max_t C_t^\dagger(x)$, and $\overline C^\dagger(x):=\prod_{t=1}^L C_t^\dagger(x)$. With oracle thresholds $\beta_t=(1-1/L)/C_t^\dagger(x)$, the guarantees in Theorem~\ref{main theorem} and Corollary~\ref{cor:complexity} extend under Assumption~\ref{assump:reward-err} and the same sample-size conditions, with the optimal path and its coverage coefficients replaced by those of the witness. The filtering and ranking arguments track this path, and, on its survival event, the selected response $\hat y$ satisfies $\hat r(x,\hat y)\ge\hat r(x,y^\dagger)$, so the reward-error argument applies with $y^\dagger$ as an optimal comparator. This extension does not exploit the aggregate probability mass or shared-prefix structure of all correct trajectories; sharper bounds that do so remain an interesting direction.
\end{remark}

\subsection{Prefix Competitiveness and Its Relaxation}
Assumption~\ref{assump:prefix} formalizes alignment between the likelihood of an optimal prefix and the final outcome reward. A correct trajectory can be difficult to sample in full because of limited token-level coverage or a long horizon, while its prefixes still have a likelihood advantage over competing prefixes. This is the regime in which likelihood-based pruning can exploit token-level coverage without sampling the entire trajectory independently.

When the base policy instead assigns higher likelihood to suboptimal prefixes, intermediate pruning can irreversibly discard the optimal trajectory. Proposition~\ref{prop:necessity} exhibits such an instance with a constant failure probability for every sample size. Thus, increasing the sampling budget alone cannot compensate for arbitrary misalignment between policy likelihood and outcome reward. Sequence-level methods such as Best-of-$N$ or majority voting may be more competitive in these regimes because they postpone selection until complete responses are available \citep{huang2024self}; the relative benefit of beam search depends on the instance.

\begin{remark}[Relaxing prefix competitiveness]
For reward-based CF-Beam, the gap condition can allow a limited number of exceptions. Suppose that, at each depth, at most $k-1$ distinct competing prefixes violate the gap inequality relative to the optimal prefix, where $1\le k\le b$, and every other competing prefix satisfies the same lower bound $\kappa$. If the optimal prefix passes filtering but is pruned, at least $b-k+1$ competitors satisfying the gap inequality must match or exceed its empirical score. Restricting the probability-mass sum in Lemma~\ref{lem: E_b} to these competitors gives the same ranking bound with $b$ replaced by $b-k+1$. The filtering and selected-output bounds remain unchanged, so Theorem~\ref{main theorem} extends with this replacement in the ranking term. We use the top-1 condition in the main text for simplicity. The reward-free final selection in Corollary~\ref{corollary: bs-mv alg} still requires the top-1 gap: a higher-likelihood suboptimal response could otherwise be selected even with exact scores.
\end{remark}

\subsection{Sampling Budget and Candidate Processing}
\begin{remark}[Beam expansion and query cost]
At each depth, Algorithm~\ref{alg:te-beam} samples $N$ next tokens for each active prefix, applies confidence filtering, and then retains at most $b$ candidate prefixes. Consequently, $|\mathcal B_t|\le b$ and, since the initial beam contains only the empty prefix, the total query budget satisfies $B(x)\le N+(L-1)bN$. Before top-$b$ pruning, the candidate pool at depth $t$ contains at most $N|\mathcal B_{t-1}|\le bN$ children; confidence filtering can only shrink this pool and introduces no additional branching. Thus, for fixed $N$ and $b$, both the number of sampling queries and the number of processed candidates grow linearly in $L$, whereas the full output space has size $|\mathcal A|^L$, which makes exhaustive search intractable \citep{stahlberg2019nmt}.
\end{remark}

These bounds count token-level policy draws, as defined in Section~\ref{sec: preliminary}. They are not direct measures of neural-network forward passes, wall-clock latency, or KV-cache memory. For instance, when the model is run locally, the $N$ draws at a common prefix can be generated from a single forward pass, whereas under pure API access each draw may constitute a separate call; actual costs further depend on batching and cache management. Comparisons with tree-search methods such as Tree-of-Thoughts or MCTS-based decoding \citep{yao2023tree,feng2023alphazero,wu2024inference} therefore require specifying both the search configuration and the cost measure.

\section{Proof of Theorem \ref{th:vanilla_beam}}\label{Appendix:Counter example for not pruning beta}

\begin{proof}[Proof of Theorem \ref{th:vanilla_beam}]
    We consider the following LLM instance, where $50 \leq C_{i}^{\star}(x) := C^{\star}(x) \ll {(|\mathcal{A}|)^{1/4}}$.
    For time step $t=1$, there exists a unique optimal action $a^\star_1$ with 
    \[
    \pi_{\text{ref}}(a_1^\star\mid x)=\frac{1}{C^\star(x)},
    \]
    and all other actions $a^\prime\in\mathcal{A}\backslash \{a_1^\star\}$ have equal probability 
    \[
    \pi_{\text{ref}}(a^\prime|x)=\frac{1-1/C^\star(x)}{|\mathcal{A}|-1}.
    \]
    For time step $t=2$, conditioned on the optimal prefix $a_1^\star$, the next-step distribution satisfies
    \[
    \pi_{\text{ref}}(a_2^\star\mid a_1^\star,x)=\frac{1}{C^\star(x)},\quad\pi_{\text{ref}}(a_2^{\prime}\mid a_1^\star,x)=\frac{1-1/C^\star(x)}{|\mathcal{A}|-1},
    \]
    where $a_2^\prime\in\mathcal{A}\backslash \{a_2^\star\}$. Otherwise, if conditioned on any suboptimal prefix $a_1^\prime\neq a_1^\star$, there exist two actions $a^\star_{2,1}$ and $a_{2,2}^{\star}$, such that 
    \[
    \pi_{\text{ref}}(a_{2,1}^\star\mid a_1',x)=\pi_{\text{ref}}(a_{2,2}^\star\mid a_1',x)=\frac{1}{2},\quad\pi_{\text{ref}}(a_2^{\prime}\mid a_1',x)=0,
    \]
    where $a_2^\prime\in \mathcal{A}\backslash\{a_{2,1}^\star,a_{2,2}^\star\}$.
    Notice that this constructed $\pi_{\text{ref}}$ does not violate our Assumption \ref{assump:prefix}, since $\pi_{\text{ref}}(s^\star_2)=1/(C^\star(x))^2  \geq 1/(2(|\mathcal{A}|-1))>\pi_{\text{ref}}(s_2^\prime)$ for any $s_2' \neq s^\star_2$.
    For time step $t>2$, we let $\pi_{\text{ref}}(\cdot|s_{t-1},x)$ be any distribution such that Assumption \ref{assump:prefix} holds. 
    \paragraph{Case 1: $N\leq 34$.} Under this case, we show that the probability that the optimal prefix will not even be sampled is at least $1/2$:
    \begin{align*}
        \Pr(y^\star \text{ is eliminated})&\geq \Pr(y^\star \text{ is eliminated in the first round})\\
        &\geq \left(1-\frac{1}{C^\star(x)}\right)^N\geq \left(\frac{49}{50}\right)^{34}>0.5.
    \end{align*}

    \paragraph{Case 2: $N\geq 35$.} Before proceeding, we define the following events. 
Let $E_1$ denote the event that the optimal trajectory $s^\star$ is eliminated in the first round, and let $E_2$ denote the event that $s^\star$ is eliminated in the second round. 
Let $\mathcal{B}_1$ be the top-$b$ beam (candidate set) retained after the first round, and define $B_\star \;:=\; \bigl\{\, |\mathcal{B}_1| = b\}$, which implies the event that we sample at least $b$ distinct trajectories from $x$.

    By the definition of our events, we have 
    \begin{align*}
        \Pr(s^\star \text{ is eliminated})&=\Pr(E_1)+\Pr(E_1^c)\Pr(E_2|E_1^c)\\
        &\geq \Pr(E_1)+\Pr(E_1^c)\Pr((E_2\cap B_\star)|E_1^c)\\
        &= \Pr(E_1)+\Pr(E_1^c)\Pr(B_\star|E_1^c)\Pr(E_2|( B_\star\cap E_1^c))\\
        &\geq \Pr(E_1)\Pr(B_\star|E_1^c)\Pr(E_2|( B_\star\cap E_1^c))\\ 
        &\quad +(1-\Pr(E_1))\Pr(B_\star|E_1^c)\Pr(E_2|( B_\star\cap E_1^c))\\
        &= \Pr(B_\star|E_1^c)\Pr(E_2|( B_\star\cap E_1^c)).
    \end{align*}
We next bound each of the term correspondingly.
\paragraph{2.1 Bounding $\Pr(B_\star\mid E_1^c)$.}
Let 
\[
K:=\#\{\text{samples equal }a_1^\star\}\sim\mathrm{Bin}(N,p),\qquad p:=\pi_{\text{ref}}(a_1^\star\mid x)=1/C^\star.
\]
We first justify that suboptimal actions are (essentially) all distinct when the action space is large.
Conditioned on $K$, there are $N-K$ suboptimal draws. Since by our construction, the suboptimal mass is sufficiently spread out
(e.g., $\max_{a\neq a_1^\star}\pi_{\text{ref}}(a\mid x)\le 1/(|\mathcal{A}|-1)$), then by a union bound,
\[
\Pr\!\left(\exists\,\text{a repeated suboptimal action}\mid K\right)
\;\le\;
\binom{N-K}{2}\max_{a\neq a_1^\star}\pi_{\text{ref}}(a\mid x)
\;\le\;
\frac{N(N-1)}{2(|\mathcal{A}|-1)}.
\]
Hence, under the regime $|\mathcal{A}|\gg N^2$, the above collision probability is negligible.
In the sequel, we work on this high-probability ``no-collision'' event, so that the $N-K$ suboptimal samples correspond to $N-K$ \emph{distinct} suboptimal actions.

\smallskip
\noindent\textbf{Step 1: Relating $E_1^c$ and $K$.}
Under the no-collision condition, every suboptimal action appears at most once, and hence has empirical frequency $1/N$.
If $K\ge 2$, then $a_1^\star$ appears at least twice and thus has empirical frequency at least $2/N$, strictly larger than any suboptimal action.
Therefore, $a_1^\star$ must be retained in the top-$b$ set after the first round, i.e.,
\[
\{K\ge 2\}\subseteq E_1^c.
\]

\smallskip
\noindent\textbf{Step 2: Characterizing $B_\star^c$.}
Recall $B_\star$ is the event that the first-round beam $\mathcal{B}_1$ is fully populated by $b$ \emph{distinct} candidates.
On the no-collision event, there are exactly $N-K$ distinct suboptimal candidates available. 
Since one slot is occupied by the optimal prefix (on $E_1^c$), the beam fails to be fully populated only if there are fewer than $b-1$ distinct suboptimal candidates, namely,
\[
B_\star^c \;\subseteq\; \{\,N-K<b-1\,\}.
\]
Consequently,
\[
\Pr(B_\star^c\mid E_1^c)
=\frac{\Pr(B_\star^c\cap E_1^c)}{\Pr(E_1^c)}
\;\le\;
\frac{\Pr(N-K<b-1)}{\Pr(E_1^c)}
\;\le\;
\frac{\Pr(N-K<b-1)}{\Pr(K\ge 2)},
\]
where the last inequality follows because $\{K\ge 2\}\subseteq E_1^c$.

\smallskip
\noindent\textbf{Step 3: Bounding the ratio.}
If $b\le N/2$, then $N-K<b-1$ implies $K\ge N-b+2\ge \lceil N/2\rceil$. Hence
\[
\Pr(B_\star^c\mid E_1^c)\le \frac{\Pr(K\ge \lceil N/2\rceil)}{\Pr(K\ge 2)}.
\]
For the numerator, since $p<1$,
\[
\Pr(K\ge \lceil N/2\rceil)
=\sum_{k=\lceil N/2\rceil}^N\binom{N}{k}p^k(1-p)^{N-k}
\le \sum_{k=\lceil N/2\rceil}^N\binom{N}{k}p^k
\le p^{N/2}\sum_{k=0}^N\binom{N}{k}
=2^N p^{N/2}.
\]
For the denominator, we lower bound $\Pr(K\ge 2)$ by a single term:
\[
\Pr(K\ge 2)\ge \Pr(K=2)=\binom{N}{2}p^2(1-p)^{N-2}.
\]
Combining the above together yields
\[
\Pr(B_\star^c\mid E_1^c)
\le
\frac{2^N p^{N/2}}{\binom{N}{2}p^2(1-p)^{N-2}}
=
\frac{2^N}{\binom{N}{2}}\cdot\frac{p^{\frac{N}{2}-2}}{(1-p)^{N-2}}
=:f(N,p).
\]
In particular, for any fixed $p\le 1/50$, $f(N,p)$ is monotonically decreasing with $N$. For any fixed $N$, $f(N,p)$ is monotonically increasing with $p$.  
Then we have  
\[
f(N,p)\le f(35,p) \leq f(35,1/50) \leq \frac{2^{35}}{\binom{{35}}{2}}\cdot
\frac{(1/50)^{\frac{{35}}{2}-2}}{(49/50)^{{35}-2}}\leq 5.21\times 10^{-19}.
\]
As a result
\[
\Pr(B_\star^c\mid E_1^c)\le f(35)\leq 5.21\times 10^{-19}.
\]
Therefore,
\[
\Pr(B_\star\mid E_1^c)=1-\Pr(B_\star^c\mid E_1^c)\ge 1-5.21\times 10^{-19}\approx 1.
\]

\smallskip
\noindent\textbf{Interpretation.}
Under a large action space, suboptimal samples are almost surely all distinct. So the only way for the top-$b$ beam $\mathcal{B}_1$ to \emph{not} be fully populated by $b$ distinct candidates is that the optimal action $a_1^\star$ is sampled \emph{too many times} (so that $N-K$ is too small to provide $b-1$ distinct suboptimal candidates). Since $K\sim\mathrm{Bin}(N,p)$ with small $p=1/C^\star$, such a large-deviation event $K\ge N/2$ has exponentially small probability, leading to $\Pr(B_\star\mid E_1^c)\approx 1$, which implies that we sample at least $b$ distinct trajectories from $x$ almost surely.

\paragraph{2.2 Bounding $\Pr(E_2\mid B^\star\cap E_1^c)$.}
For the analysis in this part, we introduce an auxiliary path-likelihood quantity
\[
\widetilde{Q}(s_t)\;:=\;\prod_{i=1}^{t}\hat{\pi}(a_i\mid s_{i-1},x),
\]
which differs from the additive score $Q(\cdot)$ used by the algorithm in
Section~\ref{sec:algorithm} and is used solely within this proof. We further
define the good event $\mathcal{E}$ as
\[
\mathcal{E}
:=
\Big\{\,\widetilde{Q}(s_2^\star)<\frac{1}{4N}\,\Big\}
\cap
\Big\{\exists\, \text{at least } b \text{ suboptimal trajectories }
s_2'\in\mathcal{S}_2
\text{ such that }
\widetilde{Q}(s_2')\geq\frac{1}{4N}
\Big\},
\]
where $\mathcal{S}_2$ is the set of all trajectories generated by $\mathcal{B}_1$.

We first control the score of the optimal trajectory under the same conditioning
used in this subsection. Let
\[
K_1:=N\hat{\pi}(a_1^\star\mid x),
\qquad
K_2:=N\hat{\pi}(a_2^\star\mid a_1^\star,x).
\]
The event $B^\star\cap E_1^c$ is determined by the first-round samples, whereas
$K_2$ is generated from the second-round samples at the optimal prefix. Hence,
conditional on $K_1$, the random variable $K_2$ remains independent of
$B^\star\cap E_1^c$ and satisfies
$
K_2\sim \mathrm{Bin}\!\left(N,{1}/{C^\star(x)}\right).
$
Moreover,
\[
\widetilde{Q}(s_2^\star)
=
\hat{\pi}(a_1^\star\mid x)
\hat{\pi}(a_2^\star\mid a_1^\star,x)
=
\frac{K_1K_2}{N^2}.
\]
Therefore, by Markov's inequality,
\begin{align*}
&\Pr\!\left(
\widetilde{Q}(s_2^\star)\geq \frac{1}{4N}
\,\middle|\,
B^\star\cap E_1^c
\right)=
\Pr\!\left(
K_1K_2\geq \frac{N}{4}
\,\middle|\,
B^\star\cap E_1^c
\right)\leq
\frac{
\mathbb{E}\!\left[K_1K_2\mid B^\star\cap E_1^c\right]
}{N/4}.
\end{align*}
It remains to control the conditional expectation in the numerator. Since
$K_2$ is independent of the first-round samples and
$\mathbb{E}[K_2]=N/C^\star(x)$, we have
\[
\mathbb{E}\!\left[K_1K_2\mid B^\star\cap E_1^c\right]
=
\mathbb{E}\!\left[K_1\mid B^\star\cap E_1^c\right]\frac{N}{C^\star(x)}.
\]
On the no-collision event considered in Section~2.1, the event
$E_1^c$ can only add outcomes with $K_1=1$ beyond the event $\{K_1\ge 2\}$,
while the event $B^\star$ only removes outcomes with too many repeated samples
of the optimal action. Hence
\[
\mathbb{E}\!\left[K_1\mid B^\star\cap E_1^c\right]
\leq
\mathbb{E}\!\left[K_1\mid K_1\geq 2\right].
\]
For a binomial random variable, conditioning on at least two successes is
stochastically dominated by adding two deterministic successes to an independent
binomial count over the remaining trials. Thus
\[
\mathbb{E}\!\left[K_1\mid K_1\geq 2\right]
\leq
2+\frac{N}{C^\star(x)}
\leq
2+\frac{C^\star(x)}{16}
\leq
\frac{C^\star(x)}{9},
\]
where the second inequality uses $N\leq (C^\star(x))^2/16$, and the last one
uses $C^\star(x)\geq 50$. Therefore,
\[
\mathbb{E}\!\left[K_1K_2\mid B^\star\cap E_1^c\right]
\leq
\frac{N}{9},
\]
and consequently
\begin{align}
\Pr\!\left(
\widetilde{Q}(s_2^\star)<\frac{1}{4N}
\,\middle|\,
B^\star\cap E_1^c
\right)
\geq 1-\frac{4}{9}
=
\frac{5}{9}.
\label{eq:opt-score-small-cond}
\end{align}

Next, we control the second event in $\mathcal{E}$. For any suboptimal prefix
$a\in\mathcal{B}_1$, conditioned on $B^\star\cap E_1^c$, its second-step
distribution has two possible offspring,
$s_{2,1}'=(a,a_{2,1}^\star)$ and $s_{2,2}'=(a,a_{2,2}^\star)$, satisfying
\[
\pi_{\rm ref}(a_{2,1}^\star\mid a,x)
=
\pi_{\rm ref}(a_{2,2}^\star\mid a,x)
=
\frac{1}{2}.
\]
We can bound the probability that both generated trajectories have sufficiently
large likelihood as follows:
\begin{align*}
&\Pr\!\left(
\widetilde{Q}(s_{2,1}')>\frac{1}{4N},
\widetilde{Q}(s_{2,2}')>\frac{1}{4N}
\,\middle|\,
B^\star\cap E_1^c
\right)\\
&=
\Pr\!\left(
\hat{\pi}(a_{2,1}^\star\mid a,x)\hat{\pi}(a\mid x)>\frac{1}{4N},
\hat{\pi}(a_{2,2}^\star\mid a,x)\hat{\pi}(a\mid x)>\frac{1}{4N}
\,\middle|\,
B^\star\cap E_1^c
\right)\\
&\geq
\Pr\!\left(
\hat{\pi}(a_{2,1}^\star\mid a,x)>\frac{1}{4},
\hat{\pi}(a_{2,2}^\star\mid a,x)>\frac{1}{4}
\right)\\
&=
\Pr\!\left(
\frac{1}{4}<\hat{\pi}(a_{2,1}^\star\mid a,x)<\frac{3}{4}
\right)\\
&=
1-\Pr\!\left(
\left|\hat{\pi}(a_{2,1}^\star\mid a,x)-\frac{1}{2}\right|
\geq\frac{1}{4}
\right)\\
&\geq 1-2\exp\!\left(-\frac{N}{8}\right),
\end{align*}
where the first inequality holds because every selected first-step prefix has
empirical probability at least $1/N$, i.e., $\hat{\pi}(a\mid x)\geq 1/N$; the
second equality follows from
$
\hat{\pi}(a_{2,1}^\star\mid a,x)
+
\hat{\pi}(a_{2,2}^\star\mid a,x)
=
1
$;
and the last inequality follows from Hoeffding's inequality applied to
$N\hat{\pi}(a_{2,1}^\star\mid a,x)\sim \mathrm{Bin}(N,1/2)$.

Then, at the second step, for each such suboptimal prefix
$s_1\in\mathcal{B}_1$, let $\mathsf{Succ}(s_1)$ denote the event that its two
offspring trajectories $s_{2,1}'$ and $s_{2,2}'$ satisfy
$
\widetilde{Q}(s_{2,1}')\geq {1}/{(4N)},\ 
\widetilde{Q}(s_{2,2}')\geq {1}/{(4N)}.
$
By construction,
\[
\Pr\!\left(\mathsf{Succ}(s_1)\mid B^\star\cap E_1^c\right)
\geq
1-2\exp\!\left(-{N}/{8}\right),
\]
and these events are independent across different suboptimal prefixes. Hence
the number of successful suboptimal prefixes conditioned on
$B^\star\cap E_1^c$ stochastically dominates
\[
Z\sim
\mathrm{Bin}\!\left(
b-1,\,
1-2\exp\!\left(-\frac{N}{8}\right)
\right).
\]
In particular, there exists $b$ suboptimal trajectories $s_2'$ such that
$\widetilde{Q}(s_2')\geq 1/(4N)$ is implied by
$\{Z\geq \lceil b/2\rceil\}$, and therefore
\begin{align}\label{eq:lemma1_reason}
&\Pr\!\left(
\exists\text{ at least } b \text{ suboptimal trajectories } s'_2
\text{ with } \widetilde{Q}(s'_2)\geq \frac{1}{4N}
\,\middle|\,
B^\star\cap E_1^c
\right) \notag\\
&\quad\geq
\Pr\!\left(Z\geq \left\lceil \frac{b}{2} \right\rceil\right)=
\Pr\!\left(
\mathrm{Bin}\!\left(
b-1,\,
1-2\exp\!\left(-\frac{N}{8}\right)
\right)
\geq \left\lceil \frac{b}{2} \right\rceil
\right).
\end{align}
By Lemma~\ref{small technical lemma} (see Appendix~\ref{sec:th1lemma}), the
above binomial tail probability is nondecreasing along the even subsequence
$\{b=2,4,6,\ldots\}$ and along the odd subsequence $\{b=3,5,7,\ldots\}$.
Consequently,
\begin{align}
&\Pr\!\left(
\mathrm{Bin}\!\left(
b-1,\,
1-2\exp\!\left(-\frac{N}{8}\right)
\right)
\geq \left\lceil \frac{b}{2} \right\rceil
\right) \notag\\
&\quad\geq
\min\!\left\{
\Pr\!\left(
\mathrm{Bin}\!\left(
1,\,
1-2\exp\!\left(-\frac{N}{8}\right)
\right)\geq 1
\right),
\right. \left.
\Pr\!\left(
\mathrm{Bin}\!\left(
2,\,
1-2\exp\!\left(-\frac{N}{8}\right)
\right)\geq 2
\right)
\right\} \notag\\
&\quad=
\left(1-2\exp\!\left(-\frac{N}{8}\right)\right)^2.
\end{align}

Using a union bound under the conditioning $B^\star\cap E_1^c$, we obtain
\[
\Pr\!\left(\mathcal{E}\mid B^\star\cap E_1^c\right)
\geq
\left(1-2\exp\!\left(-\frac{N}{8}\right)\right)^2-\frac{4}{9}.
\]
Since $N\geq 35$ in Case~2, we have
\[
\left(1-2\exp\!\left(-\frac{N}{8}\right)\right)^2-\frac{4}{9}
\geq
\left(1-2\exp\!\left(-\frac{35}{8}\right)\right)^2-\frac{4}{9}
=0.506
\]
Under the good event $\mathcal{E}$, there exist at least $b$ suboptimal
trajectories $s_2'$ such that
\[
\widetilde{Q}(s_2^\star)
<
\frac{1}{4N}
\leq
\widetilde{Q}(s_2').
\]
Therefore, the optimal response is crowded out of the top-$b$ candidates by
$b$ suboptimal trajectories. Combining this result with the bound of
$\Pr(B^\star\mid E_1^c)$ yields
\[
\Pr(s^\star\text{ is eliminated})\geq 0.5.
\]
The proof is completed.
\end{proof}

\subsection{Technical Lemmas}\label{sec:th1lemma}
In the subsection, we show that $\Pr\!\left(\mathrm{Bin}(b-1,p)\ge \frac{b}{2}\right)$ is monotone along the odd and even subsequences of $b$, respectively.
This result is useful in bounding Equation \eqref{eq:lemma1_reason} in the proof of Theorem \ref{th:vanilla_beam}.
\begin{lemma}[Monotonicity in $b$ for odd/even subsequences]\label{lem:binom-monotone-odd-even}\label{small technical lemma}
Let $N>30$ and let $p:=1-\frac{4}{N}>\frac12$.
Define, for integers $b\ge 2$,
\[
F(b):=\Pr\!\left(\mathrm{Bin}(b-1,p)\ge \frac{b}{2}\right),
\]
where for odd $b$ the threshold is understood in the usual way:
$\{\mathrm{Bin}(b-1,p)\ge b/2\}=\{\mathrm{Bin}(b-1,p)\ge \lceil b/2\rceil\}$.
Then $F(b)$ is nondecreasing along the even subsequence $\{b=2,4,6,\ldots\}$
and also nondecreasing along the odd subsequence $\{b=3,5,7,\ldots\}$.
In particular,
\[
F(2)\le F(4)\le F(6)\le \cdots,
\qquad
F(3)\le F(5)\le F(7)\le \cdots.
\]
\end{lemma}

\begin{proof}
Fix $p>\frac12$. We consider the following two cases.

\noindent\textbf{Even $b$.}
Let $b=2k$ with $k\ge 1$ and let $S_k\sim \mathrm{Bin}(2k-1,p)$. Then
$F(2k)=\Pr(S_k\ge k)$. We further have
\[
S_{k+1}=S_k+B_1+B_2,
\]
where $B_1,B_2\stackrel{i.i.d.}{\sim}\mathrm{Bern}(p)$ are independent of $S_k$.
A direct decomposition over $S_k\in\{k-1,k,\ge k+1\}$ gives
\[
F(2k+2)-F(2k)=-(1-p)^2\Pr(S_k=k)+p^2\Pr(S_k=k-1).
\]
For $S_k\sim\mathrm{Bin}(2k-1,p)$,
\[
\frac{\Pr(S_k=k)}{\Pr(S_k=k-1)}
=\frac{\binom{2k-1}{k}}{\binom{2k-1}{k-1}}\cdot\frac{p}{1-p}
=\frac{k}{k}\cdot\frac{p}{1-p}=\frac{p}{1-p}.
\]
Hence $\Pr(S_k=k-1)=\Pr(S_k=k)\frac{1-p}{p}$. Substituting yields
\[
F(2k+2)-F(2k)=\Pr(S_k=k)\,(1-p)\,(2p-1)\ge 0,
\]
which implies that $\{F(2k)\}_{k\ge 1}$ is nondecreasing.

\noindent\textbf{Odd $b$.}
Let $b=2k+1$ with $k\ge 1$ and let $T_k\sim \mathrm{Bin}(2k,p)$. Then
$F(2k+1)=\Pr(T_k\ge k+1)$. We then have
$T_{k+1}=T_k+B_1+B_2$ with the same $B_1,B_2$. Similarly,
\[
F(2k+3)-F(2k+1)=-(1-p)^2\Pr(T_k=k+1)+p^2\Pr(T_k=k).
\]
For $T_k\sim\mathrm{Bin}(2k,p)$,
\[
\frac{\Pr(T_k=k+1)}{\Pr(T_k=k)}
=\frac{\binom{2k}{k+1}}{\binom{2k}{k}}\cdot\frac{p}{1-p}
=\frac{k}{k+1}\cdot\frac{p}{1-p}.
\]
Hence, $\Pr(T_k=k)=\Pr(T_k=k+1)\frac{k+1}{k}\frac{1-p}{p}$. Therefore,
\[
F(2k+3)-F(2k+1)
=\Pr(T_k=k+1)\,(1-p)\left(\frac{k+1}{k}p-(1-p)\right).
\]
Since $p>\frac12$ implies $\frac{k+1}{k}p \ge p > 1-p$, the bracketed term is
positive. Hence, $F(2k+3)\ge F(2k+1)$ for all $k\ge 1$.
This proves monotonicity on both subsequences. Finally, the assumption $N>30$
ensures $p=1-\frac4N>\frac12$, and hence the above applies.
\end{proof}

\section{Proof of Proposition \ref{prop:necessity}}\label{appendix: proof of proposition 1}
\begin{proof}[Proof of Proposition \ref{prop:necessity}]
We exhibit an instance with horizon $L=1$ that violates
Assumption~\ref{assump:prefix} and bound the survival probability of
$y^\star$ via Markov's inequality.

\textbf{Construction.}
Let $\mathcal{A}=\{a^\star, a_1, \dots, a_{|\mathcal{A}|-1}\}$ and define
\begin{equation*}
\pi_{\mathrm{ref}}(a^\star\mid x)=\tfrac{1}{8},\qquad
\pi_{\mathrm{ref}}(a_1\mid x)=\tfrac{7}{8},\qquad
\pi_{\mathrm{ref}}(a_i\mid x)=0\;\;\text{for }i\ge 2.
\end{equation*}
Set the ground-truth reward as $r^\star(x,a^\star)=1$ and
$r^\star(x,a)=0$ for all $a\neq a^\star$, so that $y^\star=a^\star$ is
the unique optimal response.

\textbf{Violation of Assumption~\ref{assump:prefix}.}
At depth $t=1$, the suboptimal prefix $s=a_1\neq s_1^\star=a^\star$ satisfies
\begin{equation*}
\log\!\left(\frac{\pi_{\mathrm{ref}}(a^\star\mid x)}
                  {\pi_{\mathrm{ref}}(a_1\mid x)}\right)
=\log\frac{1/8}{7/8}=-\log 7<0,
\end{equation*}
so no $\kappa>0$ can satisfy the prefix-competitiveness condition on this instance.

\textbf{Survival probability.}
Fix any $N\ge 1$. Let
\begin{equation*}
X^\star := N\,\hat\pi(a^\star\mid x)\sim\mathrm{Bin}\!\big(N,\tfrac{1}{8}\big),
\qquad
X_1     := N\,\hat\pi(a_1\mid x)\sim\mathrm{Bin}\!\big(N,\tfrac{7}{8}\big).
\end{equation*}
Since $a^\star$ and $a_1$ are the only tokens with positive mass under
$\pi_{\mathrm{ref}}(\cdot\mid x)$, we have $X^\star+X_1=N$ almost surely.
With \(b=1\), the algorithm retains the sampled candidate with the largest
score \(Q(a)=Q(s_0)+\log \hat\pi(a\mid x)=\log \hat\pi(a\mid x)\).
Since the logarithm is strictly increasing, this is equivalent to retaining
the candidate with the largest empirical frequency \(\hat\pi(a\mid x)\).
Hence \(y^\star=a^\star\) is retained in the final beam only if
\(\hat\pi(a^\star\mid x)\ge \hat\pi(a_1\mid x)\), equivalently
\(X^\star\ge X_1\). Since \(X^\star+X_1=N\), this is equivalent to
\(X^\star\ge N/2\).

Applying Markov's inequality to the nonnegative
random variable $X^\star$ with $\mathbb{E}[X^\star]=N/8$,
\begin{equation*}
\Pr\!\left(X^\star\ge\tfrac{N}{2}\right)
\;\le\;\frac{\mathbb{E}[X^\star]}{N/2}
\;=\;\frac{N/8}{N/2}\;=\;\frac{1}{4}.
\end{equation*}
Consequently,
\begin{equation*}
\Pr(y^\star\text{ retained in the final beam})
\;\le\;\Pr\!\left(X^\star\ge\tfrac{N}{2}\right)
\;\le\;\frac{1}{4},
\end{equation*}
so the probability of pruning $y^\star$ is at least
$1-\tfrac{1}{4}=\tfrac{3}{4}$, completing the proof.
\end{proof}

\section{Regret Upper Bound: Proof of Theorem \ref{main theorem} and Corollary \ref{corollary: bs-mv alg}}\label{Appendix: proof of main theorem}

\subsection{Technical Lemmas: Filtering and Beam Survival}\label{Appd: Upper-techlem}
We first control the distribution of the final output and the probability of filtering out an optimal action. We then bound moments of filtered score differences and use them to control the rank of the optimal prefix. The binomial moment inequalities used below are proved in Lemma~\ref{lem:truncated-binomial-moments} in the Auxiliary Lemmas section.

Fix $x$ throughout this section. As in Definition~\ref{def: good event}, independently generate an $N$-sample histogram at every prefix of depth less than $L$ and reveal it only if the algorithm visits that prefix. Along a fixed path, these histograms belong to distinct nodes and are independent; paths with shared ancestors still have dependent scores. For a prefix $s_t=(a_1,\ldots,a_t)$, write $\mathcal F(s_t):=\{\hat\pi(a_i\mid s_{i-1},x)\ge\beta_i,\ i\in[t]\}$. On this event, $Q(s_t)=\sum_{i=1}^t\log\hat\pi(a_i\mid s_{i-1},x)$ is finite. Expressions involving inverse empirical probabilities or differences of scores are defined to be zero outside the indicated filtering events. In particular, $\mathcal E_\beta=\mathcal F(s_L^\star)$. This construction is used only for analysis and does not change the algorithm's queries.

\begin{lemma}[Output probabilities under confidence filtering]\label{lem:filtered-output}
For any thresholds $\beta_i\in(0,1]$ and any fixed $y\in\mathcal A^L$, the output $y^{\mathrm{beam}}$ of CF-Beam satisfies
\[
\Pr(y^{\mathrm{beam}}=y,\mathcal E)
\le \pi_{\mathrm{ref}}(y\mid x)\prod_{i=1}^L\beta_i^{-1}.
\]
\end{lemma}
\begin{proof}
On $\mathcal E$, the final beam is nonempty, and any output $y=(a_1,\ldots,a_L)$ has passed filtering along all its edges. Hence
\begin{align*}
\Pr(y^{\mathrm{beam}}=y,\mathcal E)
&\le\Pr(\mathcal F(y))
=\prod_{i=1}^L\Pr\bigl(\hat\pi(a_i\mid s_{i-1},x)\ge\beta_i\bigr)\\
&\le\prod_{i=1}^L\frac{\mathbb E[\hat\pi(a_i\mid s_{i-1},x)]}{\beta_i}
=\pi_{\mathrm{ref}}(y\mid x)\prod_{i=1}^L\beta_i^{-1}.
\end{align*}
The equality uses independence of the histograms along the fixed path, and the inequality follows from Markov's inequality and unbiasedness of each empirical frequency. The argument does not condition the histograms on the output-selection event.
\end{proof}

\begin{lemma}[Optimal actions pass filtering]\label{lem:good_event}
For $\beta_t=\alpha/C_t^\star(x)$ with $0<\alpha<1$,
\[
\Pr(\mathcal E_\beta^c)\le L\exp\!\left(-\frac{N(1-\alpha)^2}{2C_{\max}^{\star}}\right).
\]
\end{lemma}
\begin{proof}
For each $t$, the count $N\hat\pi(a_t^\star\mid s_{t-1}^\star,x)$ has distribution $\mathrm{Bin}(N,1/C_t^\star(x))$. Applying the lower-tail Chernoff bound in Lemma~\ref{Chernoff Inequality} with relative deviation $1-\alpha$, and then a union bound, gives
\begin{align*}
\Pr(\mathcal E_\beta^c)
&\le\sum_{t=1}^L\Pr\left(N\hat\pi(a_t^\star\mid s_{t-1}^\star,x)<\frac{\alpha N}{C_t^\star(x)}\right)\\
&\le\sum_{t=1}^L\exp\!\left(-\frac{N(1-\alpha)^2}{2C_t^\star(x)}\right)
\le L\exp\!\left(-\frac{N(1-\alpha)^2}{2C_{\max}^{\star}}\right).
\end{align*}
\end{proof}

\begin{lemma}[Moments of filtered score differences]\label{QQ comparison lemma}
Let $\beta_i=\alpha/C_i^\star(x)$ with $1/2\le\alpha<1$. For every $s_t\ne s_t^\star$ and integer $1\le r\le N\alpha/(4C_{\max}^{\star})$,
\[
\mathbb E\!\left[e^{r(Q(s_t)-Q(s_t^\star))}
\mathbf1_{\mathcal F(s_t)\cap\mathcal F(s_t^\star)}\right]
\le\left(\frac{\pi_{\mathrm{ref}}(s_t\mid x)}{\pi_{\mathrm{ref}}(s_t^\star\mid x)}\right)^r
\exp\!\left(\frac{3tC_{\max}^{\star}r^2}{N\alpha}\right).
\]
\end{lemma}
\begin{proof}
For each depth $i$, let $K_i=\lceil N\beta_i\rceil\ge N\alpha/C_{\max}^{\star}$, so $2r\le K_i/2$. Apply the positive and inverse moment bounds of Lemma~\ref{lem:truncated-binomial-moments} with order $2r$ to each edge. Independence along each fixed path yields
\begin{align*}
\mathbb E\!\left[e^{2rQ(s_t)}\mathbf1_{\mathcal F(s_t)}\right]
&\le\pi_{\mathrm{ref}}(s_t\mid x)^{2r}
\exp\!\left(\sum_{i=1}^t\frac{(2r)(2r-1)}{2(K_i-2r+1)}\right),\\
\mathbb E\!\left[e^{-2rQ(s_t^\star)}\mathbf1_{\mathcal F(s_t^\star)}\right]
&\le\pi_{\mathrm{ref}}(s_t^\star\mid x)^{-2r}
\exp\!\left(\sum_{i=1}^t\frac{(2r)(2r+1)}{2K_i}\right).
\end{align*}
All optimal edge probabilities are positive by the finiteness of $C_{\max}^{\star}$. A competing path of true probability zero has zero probability of passing filtering, so its contribution is zero. Cauchy--Schwarz now gives
\begin{align*}
&\mathbb E\!\left[e^{r(Q(s_t)-Q(s_t^\star))}
\mathbf1_{\mathcal F(s_t)\cap\mathcal F(s_t^\star)}\right]\\
&\quad\le
\left(\mathbb E[e^{2rQ(s_t)}\mathbf1_{\mathcal F(s_t)}]
\mathbb E[e^{-2rQ(s_t^\star)}\mathbf1_{\mathcal F(s_t^\star)}]\right)^{1/2}\\
&\quad\le
\left(\frac{\pi_{\mathrm{ref}}(s_t\mid x)}{\pi_{\mathrm{ref}}(s_t^\star\mid x)}\right)^r
\exp\!\left[\sum_{i=1}^t\left(
\frac{(2r)(2r-1)}{4(K_i-2r+1)}+
\frac{(2r)(2r+1)}{4K_i}\right)\right]\\
&\quad\le
\left(\frac{\pi_{\mathrm{ref}}(s_t\mid x)}{\pi_{\mathrm{ref}}(s_t^\star\mid x)}\right)^r
\exp\!\left(\sum_{i=1}^t\frac{3r^2}{K_i}\right)
\le
\left(\frac{\pi_{\mathrm{ref}}(s_t\mid x)}{\pi_{\mathrm{ref}}(s_t^\star\mid x)}\right)^r
\exp\!\left(\frac{3tC_{\max}^{\star}r^2}{N\alpha}\right).
\end{align*}
For the penultimate inequality, $K_i-2r+1\ge K_i/2$ bounds the two terms by $(2r^2-r)/K_i$ and $(r^2+r/2)/K_i$, respectively. In particular, Cauchy--Schwarz controls the common optimal score and shared-prefix randomness without requiring independent path scores.
\end{proof}

\begin{lemma}[Survival under top-$b$ pruning]\label{lem: E_b}
Suppose Assumption~\ref{assump:prefix} holds, $\beta_i=\alpha/C_i^\star(x)$ with $1/2\le\alpha<1$, and
$N\ge\frac{48C_{\max}^{\star}}{\kappa^2}\max\{2,\log(2C_{\max}^{\star})\}$.
Then
\[
\Pr(\mathcal E_\beta\cap\mathcal E_b^c)
\le\frac{2C_{\max}^{\star}}{b}
\exp\!\left(-\frac{\kappa^2N}{48C_{\max}^{\star}}\right).
\]
More precisely, let $Z_t$ count the filtered nonoptimal prefixes in the full depth-$t$ tree whose scores are at least $Q(s_t^\star)$, and set $Z_t=0$ if $\mathcal F(s_t^\star)$ fails. For every integer $m\ge1$,
\begin{equation}
\Pr(Z_t\ge m)\le\frac1m
\left[C_{\max}^{\star}\exp\!\left(-\frac{\kappa^2N}{48C_{\max}^{\star}}\right)\right]^t.
\label{eq:full-tree-rank}
\end{equation}
\end{lemma}
\begin{proof}
For any admissible integer $r\ge1$ in Lemma~\ref{QQ comparison lemma}, each competitor counted by $Z_t$ contributes at least one to the corresponding likelihood-ratio moment. Since the depth-$t$ tree is finite, summing the lemma gives
\begin{align*}
\mathbb E[Z_t]
&\le\sum_{s_t\ne s_t^\star}
\mathbb E\!\left[e^{r(Q(s_t)-Q(s_t^\star))}
\mathbf1_{\mathcal F(s_t)\cap\mathcal F(s_t^\star)}\right]\\
&\le\exp\!\left(\frac{3tC_{\max}^{\star}r^2}{N\alpha}\right)
\sum_{s_t\ne s_t^\star}
\left(\frac{\pi_{\mathrm{ref}}(s_t\mid x)}{\pi_{\mathrm{ref}}(s_t^\star\mid x)}\right)^r\\
&\le\exp\!\left(\frac{3tC_{\max}^{\star}r^2}{N\alpha}-(r-1)t\kappa\right)
\frac{\sum_{s_t\ne s_t^\star}\pi_{\mathrm{ref}}(s_t\mid x)}{\pi_{\mathrm{ref}}(s_t^\star\mid x)}\\
&\le\left[C_{\max}^{\star}
\exp\!\left(-(r-1)\kappa+\frac{3C_{\max}^{\star}r^2}{N\alpha}\right)\right]^t.
\end{align*}
The third line uses Assumption~\ref{assump:prefix}, which implies $\pi_{\mathrm{ref}}(s_t\mid x)\le e^{-t\kappa}\pi_{\mathrm{ref}}(s_t^\star\mid x)$, and $r-1\ge0$. The last line uses total prefix probability at most one and $\pi_{\mathrm{ref}}(s_t^\star\mid x)^{-1}=\prod_{i=1}^t C_i^\star(x)\le(C_{\max}^{\star})^t$. Thus the sum is controlled by true probability mass rather than by the number of adaptively selected candidates.

Choose $r=\lfloor N\alpha\kappa/(6C_{\max}^{\star})\rfloor$. The sample-size condition and $1/2\le\alpha<1$, $0<\kappa\le1$ imply $N\alpha\kappa^2/C_{\max}^{\star}\ge48$ and $1\le r\le N\alpha/(4C_{\max}^{\star})$. Using the two bounds on the floor function,
\begin{align*}
-(r-1)\kappa+\frac{3C_{\max}^{\star}r^2}{N\alpha}
&\le-\frac{N\alpha\kappa^2}{12C_{\max}^{\star}}+2\kappa\le-\frac{N\alpha\kappa^2}{24C_{\max}^{\star}}
\le-\frac{N\kappa^2}{48C_{\max}^{\star}}.
\end{align*}
The middle inequality follows from $N\alpha\kappa^2/C_{\max}^{\star}\ge48\ge48\kappa$. Markov's inequality applied to $Z_t$ proves \eqref{eq:full-tree-rank}.

On $\mathcal E_\beta$, if the optimal path is first discarded by beam pruning at depth $t$, it is a generated candidate and at least $b$ other generated candidates have scores at least its score. These candidates are among the full-tree filtered prefixes, so $Z_t\ge b$. Counting all ties against the optimal prefix makes this implication valid for any tie-breaking rule. Since the sample-size condition also gives $C_{\max}^{\star}e^{-N\kappa^2/(48C_{\max}^{\star})}\le1/2$, a union bound and a geometric sum yield
\begin{align*}
\Pr(\mathcal E_\beta\cap\mathcal E_b^c)
&\le\sum_{t=1}^L\Pr(Z_t\ge b)
\le\frac1b\sum_{t=1}^L
\left[C_{\max}^{\star}e^{-N\kappa^2/(48C_{\max}^{\star})}\right]^t\\
&\le\frac{C_{\max}^{\star}e^{-N\kappa^2/(48C_{\max}^{\star})}}
{b\left(1-C_{\max}^{\star}e^{-N\kappa^2/(48C_{\max}^{\star})}\right)}
\le\frac{2C_{\max}^{\star}}{b}e^{-N\kappa^2/(48C_{\max}^{\star})}.
\end{align*}
This bounds the intersection with $\mathcal E_\beta$ directly; no concentration inequality is applied after conditioning on a future good event.
\end{proof}

\subsection{Proof of Theorem \ref{main theorem}}
\begin{proof}
For $L\ge2$, the threshold multiplier $1-1/L$ lies in $[1/2,1)$. Applying Lemmas~\ref{lem:good_event} and~\ref{lem: E_b} with $\alpha=1-1/L$ gives
\begin{align*}
\Pr(\mathcal E^c)
&=\Pr(\mathcal E_\beta^c)+\Pr(\mathcal E_\beta\cap\mathcal E_b^c)\\
&\le L\exp\!\left(-\frac{N}{2C_{\max}^{\star}L^2}\right)
+\frac{2C_{\max}^{\star}}{b}\exp\!\left(-\frac{\kappa^2N}{48C_{\max}^{\star}}\right).
\end{align*}
On $\mathcal E$, induction over the depths gives $y^\star\in\mathcal B_L$, so the final selection satisfies $\hat r(x,y^{\mathrm{beam}})\ge\hat r(x,y^\star)$. On $\mathcal E^c$, the regret is at most one because rewards lie in $[0,1]$. Consequently,
\begin{align*}
\mathrm{Reg}(x)
&\le\Pr(\mathcal E^c)
+\mathbb E\!\Bigl[\mathbf1_{\mathcal E}
\bigl(r^\star(x,y^\star)-\hat r(x,y^\star)
+\hat r(x,y^\star)-\hat r(x,y^{\mathrm{beam}})\\
&\hspace{10em}+\hat r(x,y^{\mathrm{beam}})-r^\star(x,y^{\mathrm{beam}})\bigr)\Bigr]\\
&\le\Pr(\mathcal E^c)+\epsilon_{\mathrm{opt}}(x)
+\mathbb E\!\left[\mathbf1_{\mathcal E}
\bigl|\hat r(x,y^{\mathrm{beam}})-r^\star(x,y^{\mathrm{beam}})\bigr|\right]\\
&\le\Pr(\mathcal E^c)+\epsilon_{\mathrm{opt}}(x)
+\left(\sum_{y\in\mathcal A^L}\Pr(y^{\mathrm{beam}}=y,\mathcal E)
\bigl(\hat r(x,y)-r^\star(x,y)\bigr)^2\right)^{1/2}\\
&\le\Pr(\mathcal E^c)+\epsilon_{\mathrm{opt}}(x)
+\left[\left(1-\frac1L\right)^{-L}\overline C(x)
\sum_{y\in\mathcal A^L}\pi_{\mathrm{ref}}(y\mid x)
\bigl(\hat r(x,y)-r^\star(x,y)\bigr)^2\right]^{1/2}\\
&\le\Pr(\mathcal E^c)+\epsilon_{\mathrm{opt}}(x)
+2\sqrt{\overline C(x)\epsilon_{\mathrm{RM}}^2(x)}.
\end{align*}
The second inequality uses Assumption~\ref{assump:opt} and reward maximization. The third is Cauchy--Schwarz with $\Pr(\mathcal E)\le1$, and the fourth applies Lemma~\ref{lem:filtered-output}. The last inequality follows from Assumption~\ref{assump:reward-err} and $(1-1/L)^{-L}\le4$ for integer $L\ge2$. This inequality follows because $u\log(1-1/u)$ is nondecreasing for $u\ge2$: its derivative is $\log(1-1/u)+1/(u-1)\ge0$ by $\log(1+v)\le v$. The distribution comparison is for the joint probabilities $\Pr(y^{\mathrm{beam}}=y,\mathcal E)$, so no normalization by $\Pr(\mathcal E)$ is needed.
\end{proof}

\subsection{Proof of Corollaries \ref{cor:complexity}}
\begin{proof}[Proof of Corollary \ref{cor:complexity}]
Each search-error term in \eqref{eq:main-search-error} is at most $\delta/2$ whenever
\[
N\ge\max\left\{
\frac{48C_{\max}^{\star}}{\kappa^2}\max\{2,\log(2C_{\max}^{\star})\},\quad
2C_{\max}^{\star}L^2\log\frac{2L}{\delta},\quad
\frac{48C_{\max}^{\star}}{\kappa^2}\left[\log\frac{4C_{\max}^{\star}}{b\delta}\right]_+
\right\}.
\]
The first term is the sample-size condition in Theorem~\ref{main theorem}; the other two follow by solving the two exponential inequalities. Since $C_{\max}^{\star}\ge1$,
$[\log(4C_{\max}^{\star}/(b\delta))]_+\le\log(4C_{\max}^{\star})+[\log(1/(b\delta))]_+$, so \eqref{eq:sample_complexity_bound} with a sufficiently large universal constant implies all three requirements. Substitution into \eqref{eq:main-regret} proves the corollary.
\end{proof}

\subsection{Proof of Corollary \ref{corollary: bs-mv alg}}\label{Proof of Corollary {corollary: bs-mv alg}}
\begin{proof}
Apply Lemma~\ref{lem:good_event} with $\alpha=1/2$ to obtain $\Pr(\mathcal E_\beta^c)\le L\exp(-N/(8C_{\max}^{\star}))$. Unlike reward-based selection, the final maximum-score selection must also distinguish $y^\star$ from every remaining competitor. If $\mathcal E_\beta$ holds, $Z_t<b$ at every $t<L$, and $Z_L=0$, then the optimal path survives the first $L-1$ rounds, has strictly larger score than every other filtered complete path, and is selected at the final round. Thus
\begin{align*}
\mathrm{Reg}(x)
&\le\Pr(y^{\mathrm{beam}}\ne y^\star)
\le\Pr(\mathcal E_\beta^c)
+\sum_{t=1}^{L-1}\Pr(Z_t\ge b)+\Pr(Z_L\ge1)\\
&\le L\exp\!\left(-\frac{N}{8C_{\max}^{\star}}\right)
+\frac1b\sum_{t=1}^{L-1}\left[C_{\max}^{\star}e^{-N\kappa^2/(48C_{\max}^{\star})}\right]^t
+\left[C_{\max}^{\star}e^{-N\kappa^2/(48C_{\max}^{\star})}\right]^L\\
&\le L\exp\!\left(-\frac{N}{8C_{\max}^{\star}}\right)
+\frac{2C_{\max}^{\star}}{b}\exp\!\left(-\frac{N\kappa^2}{48C_{\max}^{\star}}\right)
+\left[C_{\max}^{\star}e^{-N\kappa^2/(48C_{\max}^{\star})}\right]^L.
\end{align*}
The first inequality uses bounded rewards and optimal reward one. The next bounds follow from \eqref{eq:full-tree-rank} and the same geometric sum as in Lemma~\ref{lem: E_b}; final selection is controlled with $m=1$, while intermediate pruning is controlled with $m=b$.

For an explicit sufficient sample size, take
$N\ge96(C_{\max}^{\star}/\kappa^2)\log(6LC_{\max}^{\star}/\delta)$.
This satisfies the sample-size condition of Lemma~\ref{lem: E_b} and gives
$C_{\max}^{\star}e^{-N\kappa^2/(48C_{\max}^{\star})}\le\delta/(6L)$.
The ranking sum is then at most $\delta/(3L)$, and its final-depth power is at most $\delta/(6L)$, since $b\ge1$ and $\delta/(6L)<1$. Also $N\ge8C_{\max}^{\star}\log(3L/\delta)$, so the filtering term is at most $\delta/3$. Their sum is at most $\delta$, proving the stated sufficient sample complexity.
\end{proof}

\section{Regret Lower Bound: Proof of \Cref{regret lower bound}}
\begin{proof}[Proof of \Cref{regret lower bound}]
Fix a horizon $L$ and an action set $\mathcal{A}=\{1,2,\dots,|\mathcal{A}|\}$ with $|\mathcal{A}|\ge \frac{\log L}{\log(4/3)}+2$.
We construct an instance in which, at each time step $t\in[L]$, there is a unique optimal action along the optimal trajectory.
Without loss of generality (by relabeling actions), we denote this optimal action $a_t^\star$ by $a_{t,1}$.

Let $s_{t-1}^\star$ be the state reached by following the optimal action sequence up to time $t-1$
(i.e., the ``optimal-prefix'' state), and let $s_{t-1}$ denote any state that is \emph{not} on this optimal prefix.
We define a reference policy $\pi_{\mathrm{ref}}(\cdot\mid s,x)$ as follows:
\[
\pi_{\mathrm{ref}}(a_{t,1}\mid s_{t-1}^\star,x)=\frac{1}{|\mathcal{A}|}+\kappa^{\prime},\qquad
\pi_{\mathrm{ref}}(a_{t,i}\mid s_{t-1}^\star,x)=\frac{1}{|\mathcal{A}|}-\frac{\kappa^{\prime}}{|\mathcal{A}|-1},\ \ i=2,\dots,|\mathcal{A}|,
\]
and for any off-prefix state $s_{t-1}\neq s_{t-1}^\star$,
\[
\pi_{\mathrm{ref}}(a_{t,i}\mid s_{t-1},x)=\frac{1}{|\mathcal{A}|},\qquad i=1,\dots,|\mathcal{A}|.
\]
By construction, $\sum_{i=1}^{|\mathcal{A}|} \pi_{\mathrm{ref}}(a_{t,i}\mid s_{t-1}^\star,x)=1$, and we assume $\kappa^{\prime}$ is a sufficiently small value that will be assigned later.
Next, define the true reward $r^\star$ and the estimated reward $\hat r$ for complete trajectories.
Let $y^\star$ denote the unique optimal trajectory induced by choosing $a_{t,1}$ at every step along the optimal prefix,
and let $y\neq y^\star$ denote any other trajectory. We set
\[
r^\star(x,y^\star)=1,\qquad r^\star(x,y)=1-\delta\ \ (y\neq y^\star),
\]
while the reward estimator is misspecified in the opposite direction:
\[
\hat r(x,y^\star)=1-2\delta,\qquad \hat r(x,y)=1-\delta\ \ (y\neq y^\star),
\]
where $\delta\in(0,1/2]$ will be chosen later, which ensures all rewards lie in $[0,1]$. This construction satisfies Assumption \ref{assump:prefix}.

Thereby, under this hard instance, if the final top-$b$ response contains suboptimal responses, the reward model will output the incorrect result with probability one. According to the definition of Regret, 
\begin{align*}
    \mathrm{Reg}(x)&=r^\star(x,y^\star)-\mathbb{E}_{y\sim\pi_{\mathrm{beam}}(\cdot|x)}[r^\star(x,y)]=\delta\pi_{\mathrm{beam}}(\mathcal{Y}^{\prime}|x),
\end{align*}
where $\mathcal{Y}^\prime=\mathcal{Y}\backslash \{y^\star\}$.
Since we have $b\geq2$ and the reward model will choose a suboptimal response if the final top-$b$ set contains $y\in \mathcal{Y}^\prime$,
\[
\pi_{\text{beam}}(\mathcal{Y}^\prime|x)=1-\Pr(\text{the final top $b$ only contains } y^\star).
\]
Providing the lower bound is to providing the upper bound for $\Pr(\text{the final top $b$ only contains }y^{\star})$.

We next prove $\Pr(\text{the final top $b$ only contains } y^\star)\leq 3/4$ by cases, and hence $\pi_{\text{beam}}(\mathcal{Y}^\prime|x)\geq 1/4$ is in a constant order. 
\paragraph{Case 1: $N\leq \max\{C^\star_t(x)\}$.} Under this case, let $\kappa^\prime\leq 1/|\mathcal{A}|$. At the step where $C^\star_t(x)=C_{\max}^\star$, the probability generating the correct response can be upper bounded by
\[
1-\left(1-\frac{1}{C^\star_t}\right)^N\leq 1-\left(1-\frac{1}{C^\star_t}\right)^{C_t^\star}\leq  3/4. 
\]
where the last inequality follows by $C^\star_t=\frac{|\mathcal{A}|}{\kappa^\prime\mathcal{|A|}+1}\geq 2$.

Therefore, the optimal response $y^\star$ is excluded from the final top-$b$ candidate set with probability at least $1/4$, and hence $\Pr(\text{the final top $b$ only contains }y^\star)\le 3/4$.

\paragraph{Case 2.1: $\{N> \max\{C^\star_t(x)\}$ and $\exists\ \beta_t\geq1/C_t^\star(x)\}$.} In this case, the optimal trajectory will be eliminated with a constant probability greater than 1/4, since 
\[
\Pr\left(\text{Bin}(N, 1/ C^\star_t(x))< N\beta_t\right)\geq \Pr\left(\text{Bin}(N, 1/ C^\star_t(x))< \frac{N}{C^\star_t(x)}\right)\geq \frac{1}{4},
\]
where the last inequality follows from Lemma \ref{lem:binom-below-mean-const}. Thus we exclude the optimal response from the final top $b$ candidate set with probability at least $1/4$, and hence $\Pr(\text{the final top $b$ only contains }y^\star)\le 3/4$.

\paragraph{Case 2.2: $\{N> \max_{t\in[L]} C_t^\star(x)\}$ and $\forall t,\ \beta_t < 1/C_t^\star(x)\}$.}
In this case, every filtering threshold $\beta_t$ lies \emph{below} the optimal-trajectory probability $1/C_t^\star(x)$.
Hence, for each $t\in[L]$, there exists some $\alpha\in(0,1)$ such that
$
\beta_t < \frac{\alpha}{C_t^\star(x)}.
$
We choose $\kappa'$ so that the probability of any fixed suboptimal next-token choice at the optimal prefix
dominates the scaled optimal one, namely
\[
\alpha\Bigl(\frac{1}{|\mathcal{A}|}+\kappa'\Bigr)
\;\le\;
\frac{1}{|\mathcal{A}|}-\frac{\kappa'}{|\mathcal{A}|-1}.
\]
Equivalently,
$
\alpha\,\pi_{\text{ref}}\!\left(a_t^\star \mid s_{t-1}^\star,x\right)
\;\le\;
\pi_{\text{ref}}\!\left(a_{t,i} \mid s_{t-1}^\star,x\right),
\ \forall i\in\{2,\ldots,|\mathcal{A}|\}.
$
Now fix any round $t$ and consider any suboptimal child $s_t'$ obtained by taking corresponding $a_{t,i}$ with $i\ge 2$
from the optimal prefix $s_{t-1}^\star$. Let
\[
p:=\pi_{\text{ref}}(a_{t,i}\mid s_{t-1}^\star,x)
=\frac{1}{|\mathcal{A}|}-\frac{\kappa'}{|\mathcal{A}|-1}.
\]
By the choice of $\alpha$ and $\kappa'$, we have $\beta_t<\alpha/C_t^\star(x)\le p$.
Therefore, under $N$ i.i.d.\ draws, this suboptimal child survives the $\beta_t$-filtering with constant probability:
\[
\Pr\!\left(\mathrm{Bin}(N,p)\ge N\beta_t\right)
\;\ge\;
\Pr\!\left(\mathrm{Bin}(N,p)> N\beta_t\right)
\;\ge\;
\Pr\!\left(\mathrm{Bin}(N,p)> Np\right)
\;\ge\; \frac14,
\]
where the last inequality follows from Lemma~\ref{lemma 12}. Consequently, the probability that a single suboptimal child is eliminated at round $t$ is at most
\[
\Pr\!\left(\mathrm{Bin}(N,p)< N\beta_t\right)\le \frac34.
\]

Since $b\ge 2$, as long as at least one suboptimal trajectory remains in the beam,
the final top-$b$ set cannot consist solely of $y^\star$ unless \emph{all} suboptimal trajectories
are eliminated at some intermediate round by the $\beta_t$-filtering (so that the beam collapses to a singleton).
Formally, for each $t\in[L]$ define the event
\[
E_t
:=\left\{
\text{at round $t$, after $\beta_t$-filtering, \emph{no} suboptimal child survives}
\right\},
\qquad
E:=\bigcup_{t=1}^L E_t .
\]
Then $\{\text{the final top-$b$ set contains only }y^\star\}\subseteq E$, and hence
\[
\Pr(\text{the final top-$b$ set contains only }y^\star)\le \Pr(E)\le \sum_{t=1}^L \Pr(E_t).
\]

Next, fix a round $t$ and consider the $N$ draws from the next-token distribution at the optimal prefix $s_{t-1}^\star$.
Let $(N_{t,1},\ldots,N_{t,|\mathcal{A}|})$ denote the resulting multinomial count vector with parameters $N$ and
probabilities $(p_{t,1},\ldots,p_{t,|\mathcal{A}|})$. Since multinomial counts are negatively associated, and thus satisfy a negative lower-orthant dependence bound \cite{joag1983negative}:
\[
\Pr(E_t)
\le
\Pr\!\left(\bigcap_{i=2}^{|\mathcal A|}
\{N_{t,i}<N\beta_t\}\right)
\le
\prod_{i=2}^{|\mathcal A|}
\Pr(N_{t,i}<N\beta_t).
\]
Moreover, for each $i\ge 2$, $N_{t,i}\sim \mathrm{Bin}(N,p_{t,i})$, and by the previous bound
$\Pr(\mathrm{Bin}(N,p_{t,i})< N\beta_t)\le 3/4$. Therefore,
$
\Pr(E_t)\le \left(\frac{3}{4}\right)^{|\mathcal{A}|-1}.
$
Consequently,
\[
\Pr(E)\le \sum_{t=1}^L \Pr(E_t)
\le L\cdot \left(\frac{3}{4}\right)^{|\mathcal{A}|-1}\le \frac{3}{4},
\]
where the last inequality follows from our choice of $|\mathcal{A}|$, namely
$|\mathcal{A}|\ge \frac{\log L}{\log(4/3)}+2$.

Taking \textbf{Case 1}, \textbf{Case 2.1} and \textbf{Case 2.2} together, $\Pr(\text{the final top $b$ only contains } y^\star)\leq 3/4$.

Finally, we specify $\delta$ through the notation of reward model error. By the definition of $\epsilon_{\text{RM}}^2$, we have
\[
\epsilon_{\text{RM}}^2=\mathbb{E}_{y\sim\pi_{\mathrm{ref}}(\cdot|x)}\left[\left(r^*(x,y)-\hat{r}(x,y)\right)^2\right]=\pi_{\text{ref}}(y^\star|x)\cdot 4\delta^2=\frac{4\delta^2}{\overline{C}}.
\]
 Therefore, taking $\delta=\sqrt{\overline{C}(x)\epsilon_{\mathrm{RM}}^2(x)}/2$, the regret can be lower bounded as 
\begin{align*}
    \mathrm{Reg}(x)&=\delta\pi_{\mathrm{beam}}(\mathcal{Y}^{\prime}|x)\\
    &=\sqrt{\frac{\overline{C}(x)\epsilon_{\mathrm{RM}}^2(x)}{4}}(1-\Pr(\text{the final top $b$ only contains } y^\star))\\
    &\geq \frac{1}{4}\sqrt{\frac{\overline{C}(x)\epsilon_{\mathrm{RM}}^2(x)}{4}}.\\
\end{align*}
The proof is completed.
\end{proof}

\section{Auxiliary Lemmas}
\begin{lemma}[Truncated binomial moments]\label{lem:truncated-binomial-moments}
Let $X\sim\mathrm{Bin}(N,p)$ and let $K\in\{1,\ldots,N\}$. For every integer $1\le q\le K$,
\[
\mathbb E\!\left[(X/N)^q\mathbf1_{\{X\ge K\}}\right]
\le p^q\exp\!\left(\frac{q(q-1)}{2(K-q+1)}\right).
\]
If $p>0$, then for every integer $q\ge1$,
\[
\mathbb E\!\left[(Np/X)^q\mathbf1_{\{X\ge K\}}\right]
\le\exp\!\left(\frac{q(q+1)}{2K}\right),
\]
where the integrand is zero on $\{X<K\}$.
\end{lemma}
\begin{proof}
Write $(X)_q=X(X-1)\cdots(X-q+1)$. On $X\ge K\ge q$, the inequality $\log(1+u)\le u$ gives
\begin{align*}
\log\frac{X^q}{(X)_q}
&=\sum_{j=0}^{q-1}\log\left(1+\frac{j}{X-j}\right)
\le\sum_{j=0}^{q-1}\frac{j}{X-j}
\le\frac{q(q-1)}{2(K-q+1)},\\
\mathbb E[(X/N)^q\mathbf1_{\{X\ge K\}}]
&\le\frac{e^{q(q-1)/(2(K-q+1))}}{N^q}\mathbb E[(X)_q]
=\frac{(N)_q}{N^q}p^q e^{q(q-1)/(2(K-q+1))}\\
&\le p^q e^{q(q-1)/(2(K-q+1))}.
\end{align*}
The factorial-moment identity follows by expressing $(X)_q$ as the sum over ordered $q$-tuples of distinct Bernoulli trials of the product of their success indicators. This also covers $p=0$.

For the inverse moment, take $p>0$. Expanding the binomial probability mass function and substituting $y=x+q$ yields
\begin{align*}
\mathbb E\frac1{(X+1)\cdots(X+q)}
&=\sum_{x=0}^N\frac{N!p^x(1-p)^{N-x}}{(x+q)!(N-x)!}\\
&=\frac{N!}{(N+q)!p^q}
\sum_{y=q}^{N+q}\binom{N+q}{y}p^y(1-p)^{N+q-y}\\
&=\frac{\Pr\{\mathrm{Bin}(N+q,p)\ge q\}}{p^q(N+1)\cdots(N+q)}
\le\frac1{p^q(N+1)\cdots(N+q)}.
\end{align*}
For $p=1$, the same identity follows directly from $X=N$. On $X\ge K$, the product $\prod_{j=1}^q(1+j/X)$ is at most $\exp(q(q+1)/(2K))$. Consequently,
\begin{align*}
\mathbb E[(Np/X)^q\mathbf1_{\{X\ge K\}}]
&\le (Np)^q e^{q(q+1)/(2K)}
\mathbb E\frac1{(X+1)\cdots(X+q)}\\
&\le\frac{N^q}{(N+1)\cdots(N+q)}e^{q(q+1)/(2K)}
\le e^{q(q+1)/(2K)},
\end{align*}
which proves the inverse-moment bound without a lower bound on $p$.
\end{proof}

\begin{definition}[Bernoulli KL]\label{def: Bernoulli KL}
For $p,q\in(0,1)$, the Bernoulli KL divergence is defined as
\[
\text{KL}\!\left(p\,\middle\|\,q\right)
:= p\log\frac{p}{q}+(1-p)\log\frac{1-p}{1-q}.
\]
\end{definition}

\begin{lemma}[Markov's inequality]
Let $X\ge 0$ be a nonnegative random variable with $\mathbb{E}[X]<\infty$. 
Then for any $a>0$,
\begin{equation*}
    \mathbb{P}(X\ge a)\le \frac{\mathbb{E}[X]}{a}.
\end{equation*}
\end{lemma}

\begin{lemma}[Chernoff Inequality under Binomial Setting]\label{Chernoff Inequality}
    Let $X_i\sim \text{Ber}(p)$, $S=\sum_{i=1}^n X_i\sim\text{Bin}(n,p)$, $\mu=\mathbb E[S]$, then $\forall \delta>0$, 
    \[
    \Pr\left(S\geq(1+\delta)\mu\right)\leq\left(\frac{e^\delta}{(1+\delta)^{1+\delta}}\right)^\mu\leq\exp\left(-\frac{\mu\delta^2}{2+\delta}\right),
    \] 
    \[
    \Pr\left(S\leq(1-\delta)\mu\right)\leq\exp\left(-\frac{\mu\delta^2}{2}\right).
    \]
    A KL-form Chernoff bound is:
    \[
    \Pr\!\left(\frac{S}{n}\geq \alpha\right)\leq\exp\!\left(-n\,\operatorname{KL}(\alpha\|p)\right)\quad \text{for }\alpha>p,
    \]
    \[
    \Pr\!\left(\frac{S}{n}\leq \alpha\right)\leq\exp\!\left(-n\,\operatorname{KL}(\alpha\|p)\right)\quad \text{for }\alpha<p,
    \]
    where
    \[
    \operatorname{KL}(\alpha\|p)=\alpha\log\frac{\alpha}{p}+(1-\alpha)\log\frac{1-\alpha}{1-p}.
    \]
\end{lemma}

\begin{lemma}\label{lemma2}
Let $X_{i,1},\dots,X_{i,N}\stackrel{\text{i.i.d.}}{\sim}
\mathrm{Ber}(\pi(a_i\mid s_{i-1},x))$ for some $\pi(a_i\mid s_{i-1},x)\in(0,1)$, and define
$
\hat{\pi}(a_i\mid s_{i-1},x)
\triangleq \frac{1}{N}\sum_{j=1}^N X_{i,j}.
$
Then for any $\varepsilon\in(0,1)$,
\[
    \Pr\left(
    \left|
    \log\frac{\hat{\pi}(a_i\mid s_{i-1},x)}{\pi(a_i\mid s_{i-1},x)}
    \right|
    \geq \log\left(\frac{1}{1-\varepsilon}\right)
    \right)
    \leq
    2\exp\left(-\frac{N\pi(a_i\mid s_{i-1},x)\varepsilon^2}{3}\right).
\]

\end{lemma}

\begin{proof}
    By Lemma~\ref{Chernoff Inequality}, we have
    \[
    \Pr\!\left(
    \bigl|\hat{\pi}(a_i\mid s_{i-1},x)-\pi(a_i\mid s_{i-1},x)\bigr|
    \geq
    \varepsilon\,\pi(a_i\mid s_{i-1},x)
    \right)
    \leq
    2\exp\left(-\frac{N\pi(a_i\mid s_{i-1},x)\varepsilon^2}{3}\right).
    \]
    That is, with probability at least
    $1-2\exp\left(-\frac{N\pi(a_i\mid s_{i-1},x)\varepsilon^2}{3}\right)$, we have
    \[
    (1-\varepsilon)\pi(a_i\mid s_{i-1},x)
    \leq
    \hat{\pi}(a_i\mid s_{i-1},x)
    \leq
    (1+\varepsilon)\pi(a_i\mid s_{i-1},x).
    \]
    This implies
    \[
    \log\frac{\hat{\pi}(a_i\mid s_{i-1},x)}{\pi(a_i\mid s_{i-1},x)}
    \in
    [\log(1-\varepsilon),\log(1+\varepsilon)].
    \]
    Hence,
    \[
    \left|
    \log\frac{\hat{\pi}(a_i\mid s_{i-1},x)}{\pi(a_i\mid s_{i-1},x)}
    \right|
    \leq
    \max\{-\log(1-\varepsilon),\log(1+\varepsilon)\}
    =
    \log\left(\frac{1}{1-\varepsilon}\right).
    \]
    Therefore,
    \[
    \Pr\left(
    \left|
    \log\frac{\hat{\pi}(a_i\mid s_{i-1},x)}{\pi(a_i\mid s_{i-1},x)}
    \right|
    \geq
    \log\left(\frac{1}{1-\varepsilon}\right)
    \right)
    \leq
    2\exp\left(-\frac{N\pi(a_i\mid s_{i-1},x)\varepsilon^2}{3}\right).
    \]
\end{proof}

\begin{lemma}\label{lemma: KL property}
Let $\alpha\in(0,1)$ and $p\in(0,1)$. Then the following two lower bounds hold:
\begin{equation}\label{eq:kl1}
\mathrm{KL}(p\|\alpha p)\geq\left(\log(1/\alpha)-(1-\alpha)\right)p,
\end{equation}
\begin{equation}\label{eq:kl2}
\mathrm{KL}(\alpha p\|p)\geq\left((1-\alpha)-\alpha\log(1/\alpha)\right)p.
\end{equation}
\end{lemma}

\begin{proof}
We prove \Cref{eq:kl1,eq:kl2} separately.

\textbf{Proof of \Cref{eq:kl1}.} By the definition of Bernoulli KL,
\[
\mathrm{KL}(p\|\alpha p)=p\log\frac{p}{\alpha p}+(1-p)\log\frac{1-p}{1-\alpha p}
= p\log\frac{1}{\alpha}-(1-p)\log\frac{1-\alpha p}{1-p}.
\]
Let $u:=\frac{(1-\alpha)p}{1-p}>0$, then
\[
\frac{1-\alpha p}{1-p}=1+u.
\]
Using $\log(1+u)\le u$, we obtain
\[
-(1-p)\log\frac{1-\alpha p}{1-p}=-(1-p)\log(1+u)\ge -(1-p)u=-(1-\alpha)p.
\]
Combining the last display with the expression of $\mathrm{KL}(p\|\alpha p)$ yields
\[
\mathrm{KL}(p\|\alpha p)\ge \bigl(\log(1/\alpha)-(1-\alpha)\bigr)p.
\]

\textbf{Proof of \Cref{eq:kl2}.} Again by the definition of Bernoulli KL,
\[
\mathrm{KL}(\alpha p\|p)=\alpha p\log\frac{\alpha p}{p}+(1-\alpha p)\log\frac{1-\alpha p}{1-p}
= \alpha p\log\alpha+(1-\alpha p)\log\frac{1-\alpha p}{1-p}.
\]
Let $x:=\frac{(1-\alpha)p}{1-p}>0$, so that
\[
\frac{1-\alpha p}{1-p}=1+x.
\]
Using $\log(1+x)\ge \frac{x}{1+x}$, we get
\[
(1-\alpha p)\log\frac{1-\alpha p}{1-p}=(1-\alpha p)\log(1+x)\ge (1-\alpha p)\frac{x}{1+x}.
\]
Note that
\[
\frac{x}{1+x}
=\frac{\frac{(1-\alpha)p}{1-p}}{1+\frac{(1-\alpha)p}{1-p}}
=\frac{(1-\alpha)p}{1-\alpha p}.
\]
Hence
\[
(1-\alpha p)\frac{x}{1+x}=(1-\alpha)p.
\]
Therefore,
\[
\mathrm{KL}(\alpha p\|p)\ge \alpha p\log\alpha+(1-\alpha)p
=\bigl((1-\alpha)-\alpha\log(1/\alpha)\bigr)p.
\]
This completes the proof.
\end{proof}

\begin{lemma}\label{lemma 12}
Let $X\sim\mathrm{Bin}(N,p)$ with $N\ge 1$ and $p\in(0,1)$.
If
\[
p\ge \frac{c}{N},
\qquad\text{where } c:=\ln(4/3)=0.28768\ldots,
\]
then
\[
\Pr\!\bigl(X>\mathbb{E}[X]\bigr)\ge \frac14.
\]
\end{lemma}

\begin{proof}
See Theorem~1 of \citet{pinelis2021best}.
\end{proof}

\begin{lemma}\label{lem:binom-below-mean-const}
Let $X\sim\mathrm{Bin}(N,p)$ with $N\ge 1$ and $p\in(0,1/2]$.
Then
\[
\Pr\!\bigl(X<\mathbb{E}[X]\bigr)=\Pr(X<Np)\ge \frac14.
\]
\end{lemma}

\begin{proof}
Let $q:=1-p\in[1/2,1)$ and define $Y:=N-X$. Then $Y\sim\mathrm{Bin}(N,q)$ and $\mathbb{E}[Y]=Nq$.
Moreover,
\[
\{X<Np\}
\iff
\{N-X>N-Np\}
\iff
\{Y>Nq\}
\iff
\{Y>\mathbb{E}[Y]\},
\]
so $\Pr(X<Np)=\Pr(Y>\mathbb{E}[Y])$.

If $N=1$, then $Y\in\{0,1\}$ and $\mathbb{E}[Y]=q\in(1/2,1)$, hence
\[
\Pr(Y>\mathbb{E}[Y])=\Pr(Y=1)=q\ge \frac12.
\]
If $N\ge 2$, since $q\ge 1/2$ and $c=\ln(4/3)<1/2$, we have $q\ge 1/2\ge c\ge c/N$.
Thus the condition of Lemma~\ref{lemma 12} holds for $Y$, and so
\[
\Pr(Y>\mathbb{E}[Y])\ge \frac14.
\]
Combining both cases yields $\Pr(X<Np)=\Pr(Y>\mathbb{E}[Y])\ge 1/4$.
\end{proof}

\section{Experiment Details and Supplemental Experiments} \label{appendix: numerical experiments}
\subsection{Details of Numerical Experiments} \label{appendix: details of numerical experiments}

\paragraph{Baselines.}
We compare CF-Beam against four representative test-time methods:
(i) \emph{Vanilla beam search}, which follows the same beam pipeline as CF-Beam but without the confidence-filtering step;
(ii) \emph{Best-of-$N$} \citep{huang2025best}, which samples $N$ responses independently and selects the one with the highest reward;
(iii) \emph{Majority Voting}, which samples $N$ responses independently and selects the one with the highest frequency; and
(iv) \emph{Best-of-Majority} \citep{di2025best}, which first filters candidates via majority voting (self-consistency) and then chooses the highest-scoring response under an external reward model.


\begin{remark}[Oracle CF-Beam and Empirical CF-Beam]
The threshold used in Theorem~\ref{main theorem} is
\[
    \beta_t=\frac{1-1/L}{C_t^\star(x)}
    =
    \Bigl(1-\frac1L\Bigr)
    \pi_{\mathrm{ref}}(a_t^\star\mid s_{t-1}^\star,x),
\]
which depends on the unknown optimal token probability along the optimal
prefix. Therefore, this choice is oracle and is generally not directly
implementable in practice. In the experiments, we include this version as
\emph{Oracle CF-Beam}, which serves as a diagnostic reference for the
theory-prescribed threshold.

For a practical implementation, we instead use a prefix-adaptive empirical
threshold
\[
    \beta_t(s_{t-1})
    =
    \gamma\,\hat p_{\max,t}(s_{t-1}),
    \qquad
    \hat p_{\max,t}(s_{t-1})
    :=
    \max_{a\in\mathcal A}\hat\pi(a\mid s_{t-1},x),
\]
where $\gamma\in(0,1)$ is a tuning parameter. This rule is motivated by the
fact that the oracle threshold is a constant fraction of the optimal token
probability. In likelihood-aligned regimes, where the optimal continuation is
locally competitive under the reference policy, the optimal token probability
is of the same order as the local maximum probability. Thus
$\hat p_{\max,t}(s_{t-1})$ provides a scale-adaptive proxy for the unknown
quantity $1/C_t^\star(x)$, while avoiding the need to know either
$C_t^\star(x)$ or the optimal token $a_t^\star$.

This empirical rule preserves the main purpose of the oracle threshold: it
filters out low-frequency tail tokens whose empirical probabilities are small
relative to the local mode, while retaining multiple plausible high-probability
continuations. In our experiments, we set $\gamma=0.3$ and refer to this
implementation as \emph{Empirical CF-Beam}. Figures~\ref{fig:diff_levels}
and~\ref{fig:horizon_len} show that Empirical CF-Beam achieves performance
comparable to, and sometimes slightly better than, Oracle CF-Beam.
\end{remark}

\paragraph{Experimental protocol.}
In all experiments, we use an LLM simulator specified by the collection of conditional next-token distributions $\{\pi_{\text{ref}}(\cdot \mid s, x)\}_{s \in \mathcal{A}^{\le L-1}}$, which fully characterizes the rollout distribution $\pi_{\text{ref}}(\cdot \mid x)$ over trajectories $y \in \mathcal{A}^L$ for a fixed prompt $x$. The optimal output is defined as $y^\star := \arg\max_{y \in \mathcal{A}^L} \pi_{\text{ref}}(y \mid x)$. Following the setting of Reinforcement Learning with Verifiable Reward (RLVR), we define the ground-truth reward $r^\star(x, y^\star) = 1$ and $r^\star(x, y) = 0$ for all $y \in \mathcal{A}^L \setminus \{y^\star\}$, and evaluate each algorithm by its accuracy in identifying $y^\star$. The noisy reward model $\hat{r}$ with error level $\epsilon_{\mathrm{RM}}^2 = 0.01$ is implemented as symmetric label noise: $\hat{r}(x, y)$ flips $r^\star(x, y)$ independently with probability $0.01$. For each instance, we run each algorithm for $300$ independent trials and report the average accuracy of returning the correct answer.

\paragraph{Construction of the LLM simulator.}
We construct a history-dependent token policy $\pi_{\text{ref}}(\cdot\mid s,x)$ over an action set $\mathcal{A}$ with horizon $L$, optimal-suboptimal gap $\kappa'$, and target optimal-trajectory probability $p=\frac{1}{\overline{C}(x)}\in(0,1)$. 
The construction ensures: (i) there exists a unique optimal trajectory $s^\star_L$ with $\pi_{\text{ref}}(s^\star_L|x)=p$, and 
(ii) for every state $s_{t-1}$ at depth $t-1$, the best action $a^\star_t(s_{t-1})$ satisfies the log-gap condition
$\log\pi_{\text{ref}}(s^\star_t\mid x)-\log\pi_{\text{ref}}(s_{t}\mid x)\ge 7\kappa'$ for all $s_{t}\neq s^\star_{t}$. 
To be specific, we first generate the optimal trajectory's distribution $\{C_t^\star(x)\}_{t=1}^L$, and then generate the suboptimal action's distribution at each time step $t\in[L]$. The key steps are illustrated as follows.

\textbf{Step 1 (optimal path).} Since the probability of generating the correct answer using one single rollout is $1/\overline{C}$, the geometric mean per-step probability along the optimal trajectory is $(1/\overline{C})^{1/L}$. To add randomness towards the language model, we draw $(\xi_1,\ldots,\xi_L)$ from a zero-sum Gaussian distribution (e.g., $\tilde\xi_t\sim \mathcal N(0,\sigma^2)$ i.i.d. and $\xi_t=\tilde\xi_t-\frac1L\sum_{j=1}^L\tilde\xi_j$), and set
\[
\log\frac{1}{C_t^\star(x)}=\frac{1}{L}\log p+\xi_t,
\qquad t=1,\ldots,L.
\]
The larger the value of $\sigma^2$, the greater the variability in per-step difficulty.

\textbf{Step 2 (suboptimal actions distribution).}
Since the optimal trajectory is set, we now focus on designing the distribution for suboptimal actions over optimal and suboptimal trajectories. First, for any suboptimal trajectories $s_{t-1}\neq s_{t-1}^\star$, $t\in[L]$, we assign the probability of generating local best action with $\pi(a_t^\star(s_{t-1})\mid s_{t-1},x)=\pi(a_t^\star\mid s^\star_{t-1},x)=\frac{1}{C_t^\star(x)}$. Second, for suboptimal actions on both optimal and suboptimal trajectories, we sample their probability mass vector using a Dirichlet distribution: 
\[
\left(\pi_{\mathrm{ref}}(a\mid s_{t-1},x)\right)_{a\in\mathcal{A}\setminus\{a^\star(s_{t-1})\}}\sim\mathrm{Dirichlet}(\alpha\mathbf{1}),
\]
where $\alpha$ controls the concentration of the suboptimal action distribution.
Finally, we normalize the distribution over suboptimal actions such that $\sum_{a\neq a_t^\star(s)}\pi_\mathrm{ref}(a\mid s,x)=1-\frac{1}{C_t^\star(x)}$ and it does not violate our unique optimal trajectory assumption. 
Therefore, given the basic parameters $|\mathcal{A}|$, $L$ and $p$, together with the hyperparameters $\sigma$ and $\alpha$, we construct next-token distributions $\pi_{\mathrm{ref}}(\cdot\mid s,x)$ for all prefixes $s$ of length at most $L-1$, thereby simulating the LLM sampling process.

\paragraph{Parameter setup.}\

\textbf{Figure \ref{fig:diff_levels}:} $|\mathcal{A}|$=100, $L$=10, $p=[0.01, 0.05,0.3]$, $\kappa=0.01$, $\alpha=0.5$, $\sigma=1$.\\
\textbf{Figure \ref{fig:horizon_len}:} $|\mathcal{A}|$=100, $L\in[2, 40]$, $p=0.8^L$, $\kappa=0.01$, $\alpha=0.1$, $\sigma=1$.

\paragraph{Detailed empirical findings.}
We expand here on the discussion in Section~\ref{Section: Experiment} with a finer-grained, per-panel analysis of Figures~\ref{fig:diff_levels} and~\ref{fig:horizon_len}.

\emph{Difficulty experiment (Figure~\ref{fig:diff_levels}).}
In the hardest setting ($p=0.01$, Figure~\ref{fig:diff_a}), both Oracle CF-Beam and Empirical CF-Beam achieve substantially higher accuracy than the sequence-level parallel baselines (Best-of-$N$, Majority Voting, Best-of-Majority) and Vanilla-Beam, especially under limited budgets. This is precisely the regime where sparse coverage and finite-sample tail fluctuations are most pronounced, and hence where the failure mode identified in Theorem~\ref{th:vanilla_beam} is most amplified. As the problem becomes easier (Figure~\ref{fig:diff_b}), the performance gap narrows: the sequence-level baselines catch up at moderate-to-large budgets, and the improvement of CF-Beam over Vanilla-Beam becomes more moderate. In the easiest regime ($p=0.3$, Figure~\ref{fig:diff_c}), most methods achieve comparable performance, and Vanilla-Beam nearly matches CF-Beam once the budget is sufficiently large. These trends suggest that the confidence filter is most beneficial in hard, sparse-coverage instances, while its marginal gain decreases when the base policy is already accurate enough for the optimal prefixes to be reliably sampled. This is also consistent with the gap between the sequence-level coefficient $\overline{C}$ and the worst-step coefficient $C_{\max}^{\star}$ shown in each panel: as $p$ increases, this gap becomes less pronounced, reducing the potential advantage of CF-Beam.

\emph{Horizon-length experiment (Figure~\ref{fig:horizon_len}).}
We consider a setting where, at each step, the model selects the correct action with mean probability $0.8$. As the horizon grows, the accuracy of Best-of-$N$, Best-of-Majority, and Majority Voting drops rapidly, whereas the beam-based methods are more robust. Among them, the CF-Beam variants consistently maintain the highest accuracy, indicating that confidence filtering further improves the stability of beam search in long-horizon regimes. This is consistent with our theory, which shows that the dominant sample-complexity term for CF-Beam grows only polynomially in the horizon length $L$. In contrast, sequence-level baselines can exhibit exponential dependence on $L$, causing their performance to degrade more rapidly as the horizon increases.

\begin{remark}[Finite-budget performance of Empirical CF-Beam] It may seem counterintuitive that, in Figure~\ref{fig:horizon_len}, Empirical CF-Beam slightly outperforms Oracle CF-Beam. Here, \emph{oracle} refers to knowledge of the optimal-token probability used to set the threshold; the theory-prescribed choice need not maximize accuracy at every finite budget. Writing $p_t^\star:=\pi_{\mathrm{ref}}(a_t^\star\mid s_{t-1}^\star,x)$, the oracle threshold $(1-1/L)p_t^\star$ leaves a margin of only $p_t^\star/L$ below the mean empirical frequency. As $L$ grows, this shrinking margin makes the filter more sensitive to downward sampling fluctuations. The sufficient condition in Corollary~\ref{cor:complexity} includes $N\gtrsim C_{\max}^{\star}L^2\log(2L/\delta)$ to control filtering errors, whereas the horizon experiment fixes the total sampling budget and therefore cannot increase the per-prefix sample size at this rate. With $\gamma=0.3$, the empirical threshold $\gamma\hat p_{\max,t}(s_{t-1}^\star)$ can provide a lower cutoff when $\hat p_{\max,t}(s_{t-1}^\star)$ tracks $p_t^\star$. This reduces the risk of premature elimination, while potentially admitting more suboptimal candidates, and provides a plausible explanation for the observed empirical advantage.
\end{remark}


\subsection{Additional Experiments over LLMs}
\label{appendix:llm-experiments}

To complement the simulator experiments, we evaluate CF-Beam on a matrix-multiplication task using a pretrained LLM. This experiment examines whether the empirical confidence filter improves response accuracy under a comparable sampling budget when both the reference policy and the reward model are implemented by neural networks.

\paragraph{Experimental setup.}
We evaluate Qwen3-1.7B-Base \cite{yang2025qwen3} on $500$ synthetic $4\times4$ integer matrix-multiplication problems. All methods sample from the same reference policy $\pi_{\mathrm{ref}}$, given by the base model at sampling temperature $1.3$. Completed responses are scored by AceMath-7B-RM \citep{liu2025acemath}, which supplies the outcome reward for methods that use reward-based selection. We compare Empirical CF-Beam with the same four baselines used in the simulator experiments: Best-of-$N$, Majority Voting, Best-of-Majority, and Vanilla-Beam. Both beam methods follow Algorithm~\ref{alg:te-beam}, with per-prefix sample size $N=12$ and beam width $b=2$. For Empirical CF-Beam, we use the prefix-adaptive threshold $\beta_t(s)=\gamma\hat p_{\max,t}(s)$ with $\gamma=0.3$, as defined in the preceding subsection. We report the accuracy of the final selected response and the sampling budget, using the token-level policy-query accounting defined in the main text.

\begin{table}[htbp]
    \centering
    \caption{Comparison on $500$ synthetic $4\times4$ integer matrix-multiplication problems using Qwen3-1.7B-Base. Beam methods use $N=12$ and $b=2$; Empirical CF-Beam uses $\gamma=0.3$. Budgets are reported in thousands of token-level policy queries.}
    \label{tab:llm-comparison}
    \small
    \setlength{\tabcolsep}{5pt}
    \begin{tabular}{lccccc}
        \toprule
        & Best-of-$N$ & \shortstack{Majority\\Voting} & \shortstack{Best-of-\\Majority} & Vanilla-Beam & CF-Beam \\
        \midrule
        Accuracy & 0.332 & 0.304 & 0.326 & 0.372 & \textbf{0.388} \\
        Budget ($\times10^3$) & 16.8 & 16.8 & 16.8 & 18.6 & \textbf{16.2} \\
        \bottomrule
    \end{tabular}
\end{table}

\paragraph{Comparison with baselines.}
Table~\ref{tab:llm-comparison} shows that Empirical CF-Beam achieves the highest accuracy among the evaluated methods, reaching $0.388$ with a sampling budget of $16.2\mathrm{k}$. This improves on Vanilla-Beam by $1.6$ percentage points while using a smaller budget ($16.2\mathrm{k}$ versus $18.6\mathrm{k}$). Relative to the strongest sequence-level baseline, Best-of-$N$, CF-Beam improves accuracy by $5.6$ percentage points, again with a slightly smaller sampling budget. These results provide evidence that confidence filtering can improve the accuracy--budget tradeoff in this LLM setting, consistent with the trends observed in the simulator experiments.

\paragraph{Ablation over the filtering threshold and beam width.}
We further vary $\gamma\in\{0,0.15,0.3,0.45,0.6,0.75,0.9\}$ and $b\in\{2,4,8\}$ while keeping $N=12$ fixed. Here, $\gamma=0$ disables confidence filtering and recovers Vanilla-Beam. Since the response lengths and reward-model errors depend on the chosen model and task, this ablation focuses on the two algorithmic parameters $\gamma$ and $b$.

\begin{table}[htbp]
    \centering
    \caption{Accuracy under different beam widths $b$ and filtering coefficients $\gamma$ on the same LLM task, with $N=12$. The column $\gamma=0$ corresponds to Vanilla-Beam. Bold entries mark the highest accuracy for each beam width.}
    \label{tab:llm-ablation}
    \small
    \setlength{\tabcolsep}{6pt}
    \begin{tabular}{lccccccc}
        \toprule
        & $\gamma=0$ & $0.15$ & $0.3$ & $0.45$ & $0.6$ & $0.75$ & $0.9$ \\
        \midrule
        $b=2$ & 0.372 & 0.378 & \textbf{0.388} & 0.372 & 0.318 & 0.292 & 0.290 \\
        $b=4$ & 0.482 & \textbf{0.512} & 0.472 & 0.392 & 0.374 & 0.330 & 0.302 \\
        $b=8$ & 0.536 & \textbf{0.550} & 0.480 & 0.434 & 0.364 & 0.292 & 0.300 \\
        \bottomrule
    \end{tabular}
\end{table}

Table~\ref{tab:llm-ablation} shows that the best accuracy is attained at a moderate filtering coefficient: $\gamma=0.3$ for $b=2$ and $\gamma=0.15$ for $b=4$ and $b=8$. At each beam width, the best filtered configuration improves on Vanilla-Beam. Larger thresholds lead to substantially lower accuracy, consistent with overly aggressive filtering discarding useful continuations. These observations reflect the tradeoff between removing low-frequency tail candidates and preserving plausible reasoning trajectories. Increasing the beam width also improves the best observed accuracy, from $0.388$ at $b=2$ to $0.512$ at $b=4$ and $0.550$ at $b=8$. This comparison holds the per-prefix sample size fixed; maintaining a wider beam generally requires a larger total sampling budget.





\end{document}